\documentclass[10pt]{article} % For LaTeX2e
\usepackage[preprint]{tmlr}
\usepackage{amsmath,amsfonts,bm}

\def\eqref#1{equation~\ref{#1}}
\def\1{\bm{1}}

\DeclareMathAlphabet{\mathsfit}{\encodingdefault}{\sfdefault}{m}{sl}
\SetMathAlphabet{\mathsfit}{bold}{\encodingdefault}{\sfdefault}{bx}{n}

\newcommand{\E}{\mathbb{E}}

\newcommand{\R}{\mathbb{R}}

\usepackage{hyperref}
\usepackage{url}

\usepackage{amsthm}
\usepackage{amssymb}
\usepackage{amsfonts}
\usepackage{mathtools}

\usepackage{graphicx}
\usepackage{subcaption}
\usepackage{caption}
\usepackage{xcolor}

\theoremstyle{definition}
\newtheorem{theorem}{Theorem}[section]

\newtheorem{assumption}[theorem]{Assumption} 
\newtheorem{lemma}[theorem]{Lemma}
\newtheorem{remark}[theorem]{Remark}

\newtheorem{corollary}[theorem]{Corollary}

\def \d{\mathrm d}
\def \E{\mathbb E}
\def \R{\mathbb{R}}

\title{Dimension-free error estimate for diffusion model and optimal scheduling}

\author{\name Valentin de Bortoli \thanks{Authors are listed in alphabetical order.} \email vdebortoli@google.com \\
      \addr Google DeepMind
      \AND
      \name Romuald Elie \email relie@google.com \\
      \addr Google DeepMind
      \AND
      \name Anna Kazeykina \email anna.kazeykina@universite-paris-saclay.fr\\
      \addr Universit\'e Paris Saclay\\
      \AND
      \name Zhenjie Ren \email zhenjie.ren@univ-evry.fr \\
      \addr  Universit\'e Evry - Paris Saclay
      \AND
      \name Jiacheng Zhang \email jiachengzhang@cuhk.edu.hk \\
      \addr Chinese University of Hong Kong
      }

\def\month{11}  % Insert correct month for camera-ready version
\def\year{2025} % Insert correct year for camera-ready version
\def\openreview{\url{https://openreview.net/forum?id=XXXX}} % Insert correct link to OpenReview for camera-ready version

\begin{document}

\maketitle

\begin{abstract}
Diffusion generative models have emerged as powerful tools for producing synthetic data from an empirically observed distribution. A common approach involves simulating the time-reversal of an Ornstein–Uhlenbeck (OU) process initialized at the true data distribution. Since the score function associated with the OU process is typically unknown, it is approximated using a trained neural network. This approximation, along with finite time simulation, time discretization and statistical approximation, introduce several sources of error whose impact on the generated samples must be carefully understood.
Previous analyses have quantified the error between the generated and the true data distributions in terms of Wasserstein distance or Kullback–Leibler (KL) divergence. However, both metrics present limitations: KL divergence requires absolute continuity between distributions, while Wasserstein distance, though more general, leads to error bounds that scale poorly with dimension, rendering them impractical in high-dimensional settings.
In this work, we derive an explicit, dimension-free bound on the discrepancy between the generated and the true data distributions. The bound is expressed in terms of a smooth test functional with bounded first and second derivatives. The key novelty lies in the use of this weaker, functional metric to obtain dimension-independent guarantees, at the cost of higher regularity on the test functions. As an application, we formulate and solve a variational problem to minimize the time-discretization error, leading to the derivation of an optimal time-scheduling strategy for the reverse-time diffusion. Interestingly, this scheduler has appeared previously in the literature in a different context; our analysis provides a new justification for its optimality, now grounded in minimizing the discretization bias in generative sampling.
\end{abstract}

\section{Introduction}
Generative models (GMs) have attracted substantial interest in machine learning over the past decade. Their core objective is to learn the underlying probability distribution of observed data, enabling the generation of new samples that resemble the training set. Notable classes of GMs include Variational Autoencoders (VAEs), Generative Adversarial Networks (GANs), and diffusion-based models (DMs), each employing distinct optimisation objectives to achieve faithful data synthesis and generalisation \cite{goodfellow2014generative, kingma2013auto, ho2020denoising, rezende2014stochastic}.

Score-based generative diffusion models (SGMs) have recently gained particular attention, achieving state-of-the-art results in generative modelling \cite{song2019generative, song2021scorebased}. Roughly speaking, the functioning of SGMs can be described as follows. Let $m_0$ denote the true data distribution that the model aims to learn. One first simulates an Ornstein–Uhlenbeck (OU) process starting from $m_0$; for sufficiently large times, its marginal distribution $m_t$ approaches the invariant measure of the OU process, denoted $m^*$, which is Gaussian. In the second step, a time-reversed process is simulated, initialised at $m^*$, to obtain an approximation of the original distribution $m_0$. However, the score $\nabla \ln m_t$ required for the time-reversal dynamics is unknown in practice and must be approximated, for example, by a neural network trained via score matching. Furthermore, a time-discretisation scheme is introduced to numerically simulate the reversed process.

This procedure inevitably introduces several sources of error: the \emph{statistical error}, due to replacing the true data distribution with its empirical approximation from finite samples; the \emph{finite-time simulation error}, stemming from approximating the true marginal distribution of the OU process by its invariant measure at large times; the \emph{score-matching error}, arising from training the neural network to approximate the score function; and the \emph{time-discretisation error}, incurred when simulating the reversed process numerically.

From a practical perspective, it is crucial to understand how these errors propagate through the diffusion model and impact the generated outputs. Most existing analyses of error estimation in diffusion models rely on relatively strong regularity assumptions on the score function or its estimator \cite{chen2023diffmodels, debortoli2023convergencediffusion, debortoli2021diffschrodingerbridge, chen2023improvedanalysis, lee2022polynomialcomplexity, lee2023generaldistributions}. More recent work \cite{conforti2025diffusion} has proposed introducing an early stopping rule in the backward process to derive explicit and sharp bounds on the Kullback–Leibler (KL) divergence between the data distribution and the generative model output.

In the existing literature, the discrepancy between the generated distribution and the empirical data distribution is typically evaluated using the KL divergence or the Wasserstein distance, each of which has limitations. In particular, since the empirical approximation $m_0^N$ of the learned distribution $m_0$ is not absolutely continuous, the KL divergence is not applicable to quantify the statistical error. As for the Wasserstein distance, it yields an estimate of the form $W( m_0, m_0^N ) = O( N^{-1/d} )$, where $d$ is the data dimension \cite{FournierGuillin}. This curse of dimensionality on the estimate becomes very poor in high-dimensional settings common to real-world applications.

In this paper, we propose to overcome these limitations by introducing a weaker distance, defined through a class of smooth test functions with bounded derivatives. This approach yields statistical error estimates that are well-defined for empirical measures and free from the curse of dimensionality.

The main contribution of this paper is the derivation of a \emph{dimension-free} explicit bound on the error between the true data distribution and the distribution generated by the SGM. This error is measured via a distance defined through a class of smooth test functions with bounded first and second derivatives. Achieving a dimension-free estimate comes at the cost of requiring higher regularity of the test functionals.

We begin by deriving the dimension-free error bound in the continuous-time setting. The resulting upper bound decomposes naturally into contributions from the statistical error, the finite-time simulation error, and the score-matching error incurred during the diffusion process. We then extend our analysis to incorporate a time discretisation scheme for simulating the time-reversal process. To improve the efficiency of the backward simulation, we introduce time scheduling into the time-reversal diffusion and derive the resulting error estimate as a function of both the discretisation step size and the scheduling function.

Importantly, minimising the derived error bound yields an \emph{optimal scheduler} that enhances the accuracy of the generated approximation of the data distribution. This optimisation problem is equivalent to a variational problem for the scheduler, whose continuous-time solution can be obtained explicitly. Interestingly, the scheduler thus derived coincides with a scheduling scheme that has already been proposed in the SGM literature, where it has been empirically shown to improve generative performance \cite{albergo2023stochasticinterpolants}. Other scheduling strategies have also been studied previously \cite{strasman2025noiseschedule}. Our theoretical results provide a rigorous justification for the optimality of the obtained scheduler, demonstrating that it minimises the error associated with the time discretisation of the inverse diffusion process.

\section{Background}

We first introduce diffusion model following the Stochastic Differential Equation (SDE) point of view of \cite{song2021scorebased}. We highlight that this is not the only possible presentation of diffusion models and more recent ones include Flow Matching \cite{lipman2022flow} and Stochastic Interpolant \cite{albergo2023stochastic}. However, focus on the SDE perspective as it is amenable to our theoretical investigation. 

We aim at generating samples from a distribution $m_0$. Let $m_0$ denote the true underlying data distribution, and define the empirical distribution based on $N$ available data points as $m^N_0 := \frac{1}{N} \sum_{i=1}^N \delta_{x^i_0}$, where the samples $x^i_0$ are assumed to be i.i.d.\ draws from $m_0$.
Let $X$ be forward diffusion:
\[dX_t = b(X_t) dt + dW_t.\]
Denote by $m_t: = {\rm Law}(X_t)$ given $X_0\sim m_0$, and by $m_t^N: = {\rm Law}(X_t)$ given $X_0\sim m_0^N$. 

In order to generate sample from the distribution $m_0$ we consider the \emph{time-reversal} of $X$. Under mild assumptions, see \cite{cattiaux2023time} for instance, we have that the backward process $(\overleftarrow{X}_t )_{t \in [0,T]}$ is given by the following SDE
\begin{equation*}
\label{eq:backward_eqn}
d\overleftarrow{X}_t =
\Big(- b(\overleftarrow{X}_t)  +   \frac{\alpha^2 +1}{2}\nabla \log m_{T-t}^N(\overleftarrow{X}_t)\Big)dt +\alpha d\overleftarrow{W}_t, \quad\mbox{for $\alpha\ge 0$.}
\end{equation*}
Denote by $\mu^N_t := {\rm Law}(\overleftarrow{X}_t)$ given ${\rm Law}(\overleftarrow{X}_0) = m^N_0$. In particular, note that $\mu^N_T  =m^N_0$.
Next, introduce the parametrized function $S$ to approximate the score function $\nabla\log m^N$:
\begin{equation}
\label{eq:loss_function}
   \inf_{S\in \mathcal{S}} \int_\delta^T \int \left|\nabla\log m^N_t(x) - S_t(x)\right|^2 m^N_t(dx) dt.
\end{equation} 
In practice, $S$ is parameterized with an expressive neural network and the loss is minimized using \eqref{eq:loss_function} using Stochastic Gradient descent derived optimizers. 

\section{Main result}

\subsection{Continuous-time model}\label{subsec:continuous}

In this work, we are interested in the accuracy of the generation of \eqref{eq:backward_eqn}. More precisely, we are interested in the error between the distribution of $\mu_{T-\delta}^*$, which correspond to the distribution of the generative process stopped at time $T-\delta$, and $m_0$, the underlying target data distribution. We begin by considering the case without time-discretization and extend our results to deal with the time-discretization in Section \ref{sec:discrete_time}. 

\begin{assumption}\label{assum:contraction}
   Assume that the drift $b$ is Lipschitz, that the forward diffusion admits a unique invariant measure $m^*$, and that $H(m_t|m^*)\le e^{-2\rho t} H(m_0|m^*)$ for some $\rho>0$.
\end{assumption}

\begin{assumption}\label{assum:S}
Assume that there exists a score matching function $S^*$ such that
    \begin{enumerate}
        \item $x\mapsto S^*_t(x; m^N_0)$ is $L$-Lipschitz on $t\in [\delta, T]$;
        \item The score matching error satisfies:
    \end{enumerate}
        \[\mathbb{E}\left[ \int_\delta^T \int \left|\nabla\log m^N_t(x) - S^*_t(x)\right|^2 m^N_t(dx) dt\right] \le \epsilon^2.\]
\end{assumption}

\begin{remark}
   Although \(m_t^N\) is singular at \(t=0\), the Hessian of its log-density is bounded on any interval $[\delta, T]$ with \(\delta>0\).
In the Ornstein–Uhlenbeck case \(b(x)=-x\), a direct computation gives
\[
\bigl\|\nabla^{2}\log m_t^N\bigr\|_{\mathrm{op}}
\;\le\;
\frac{1}{1-e^{-2t}}
\;+\;
\frac{e^{-2t}\,\max_{i,j}\|x_0^{i}-x_0^{j}\|^{2}}{(1-e^{-2t})^{2}}.
\]
Consequently, for any \(0<\delta\le t\le T\), the right-hand side is finite, and the bound depends only on
\(\delta\), $T$ and the initial diameter \(\max_{i,j}\|x_0^{i}-x_0^{j}\|\).
Therefore, when using \(S_t^{*}\) to approximate \(\nabla\log m_t^N\), it is justified to
require \(S_t^{*}\) to be Lipschitz in \(x\) uniformly for \(t\in[\delta,T]\).
\end{remark}

Now consider the diffusion driven by the score matching function:
\begin{equation*}
\label{eq:backward_eqn_S}
d\overleftarrow{X}^*_t = \Big(- b(\overleftarrow{X}^*_t)  +   \frac{\alpha^2 +1}{2}S^*_{T-t}(\overleftarrow{X}^*_t)\Big)dt
+\alpha d\overleftarrow{W}_t,\quad\mbox{for $\alpha\ge 0$.}
\end{equation*}
Denote by $\mu^*_t:= {\rm Law}(\overleftarrow{X}^*_t)$ given ${\rm Law}(\overleftarrow{X}^*_0) = m^*$.  We are going to estimate the  error between $\mu^*_T$ and $\mu^N_T =m^N_0$, and further obtain the  error between $\mu^*_T$ and $m_0$.

Unlike the vast majority of related literature, which estimates the error between the target measure $m_0$ and the simulated measure $\mu^*_T$ under some well- known metrics or divergences, such as Wassertein distance, total variation and Kullback–Leibler divergence, we shall consider the error through regular test functionals. As we shall see, the regularity of the test functionals helps to obtain a dimension-free upper bound on the error. 

Let $\mathcal{P}_2$ be the set of probability measures on $\mathbb{R}^d$ with finite second moment, and $G:\mathcal{P}_2 \rightarrow \mathbb{R}$ be a test functional. We say that $G$ is linearly differentiable, if there exists a linear derivative $\frac{\delta G}{\delta m}: \mathcal{P}_2\times\mathbb{R}^d\rightarrow\mathbb{R}$ such that for any $m,m'\in \mathcal{P}_2$ 
\begin{equation*}
G(m')-G(m) = \\ \int_0^1\int_{\mathbb{R}^d}\frac{\delta G}{\delta m}\big(\lambda m' +(1-\lambda)m , x\big)(m'-m)(dx)d\lambda.
\end{equation*}
Similarly, we may define the second order linear derivative $\frac{\delta^2 G}{\delta m^2}$. Further, we recall the definition of the intrinsic derivative:
\[D_m G(m,x):  = \nabla \frac{\delta G}{\delta m}(m,x),\]
where $\nabla$ denotes as usual the gradient in $x$. 

\begin{assumption}\label{assum:G}
    Assume that $G$ admits bounded   intrinsic derivatives $D_m G$ and $\nabla D_m G$, as well as bounded linear derivatives $\frac{\delta G}{\delta m}$ and $\frac{\delta^2 G}{\delta m^2}$. 
\end{assumption}

\begin{theorem}[Error for continuous time-reversed diffusion]\label{thm:continuous}
    Let $S^*$ and $G$ satisfy Assumptions \ref{assum:contraction}, \ref{assum:S} and \ref{assum:G}. Then the  error reads
\begin{equation*}
    \mathbb{E}\Big[\left|G\big(\mu^*_{T-\delta}\big) - G(m_0)\right|^2 \Big]\le C \left(\frac{1}{N} + e^{-2\rho T} + C_T\epsilon^2(\alpha^2 +1)^2+\delta^2 \right).
\end{equation*}
\end{theorem}
\begin{remark}
    Note that $S^*$ and $\mu^*_T$ are random due to the randomness of $m_0^N$.
\end{remark}

\subsection{Discrete-time model and optimal scheduling}
\label{sec:discrete_time}

We further investigate the error when we discretize the time-reversal diffusion. In order to study the optimal discretization scheme, we introduce the scheduling $(g(t))_t$ of the forward diffusion:
\[dX_t = -\dot g(t) X_t dt + \sqrt{\dot g(t)} dW_t.\] Compared to the study in Section \ref{subsec:continuous}, we restrict ourselves to the case $b(x) = -x$. Let $(m^N_t)_t$ be its marginal laws given $X_0\sim m^N_0$. Note that the invariant measure $m^*$ is Gaussian $\mathcal{N}(0, \frac12 I_d)$ in this case. 

Define
\(p^N_t:= m^N_t/m^*\), so that
\[\nabla \ln p^N_t(x):=\nabla \ln m^N_t(x) + 2x.\]
The continuous time-reversal diffusion reads:
\begin{equation*}
\d\overleftarrow{X}_t = \dot g(T-t)\Big( - \overleftarrow{X}_t  +  \nabla \log p_{T-t}^N(\overleftarrow{X}_t)\Big)\d t\\
+ \sqrt{\dot g(T-t)}\d\overleftarrow{W}_t.
\end{equation*}
Denote as in Section \ref{subsec:continuous} by $\mu^N_t := {\rm Law}(\overleftarrow{X}_t)$ given ${\rm Law}(\overleftarrow{X}_0) = m^N_T$. 
We are going to study the discretized time-reversal diffusion, where $\nabla \log p^N$ is approximated by the score function: for $t\in[t_k, t_{k+1})$ 
\begin{equation*}
\d\overleftarrow{X}^h_t = \dot g(T-t) \Big(-\overleftarrow{X}^h_t  +   S^*_{T-{t_k}}(\overleftarrow{X}^h_{t_k})\Big)\d t \\
+  \sqrt{\dot g(T-t)}\d\overleftarrow{W}_t.
\end{equation*}
with $t_0=0$, $t_{k+1}-t_k=h$, $t_K=T-\delta$ and  $\overleftarrow{X}^h_0\sim m^*$. Denote by  $\mu_t^h \coloneqq {\rm Law}(\overleftarrow{X}^h_t)$.

\begin{assumption}\label{assum:score_discrete}
Assume that:
   \begin{enumerate}
        \item $x\mapsto S^*_t(x; m^N_0)$ is $L$-Lipschitz on $t\in [\delta, T]$;
        \item  The score matching error is well-controlled
\end{enumerate}
\begin{equation}
\label{eq:dicrete score}
   \sum_{k=0}^{K-1} \int_{t_k}^{t_{k+1}}\dot g(T-t)\d t \\
   \times \int \left|\nabla\log p^N_{T-t_{k}}(x) - S^*_{T-t_{k}}(x)\right|^2 m^N_{T-t_k}(\d x)  \le \epsilon^2.
\end{equation}
\end{assumption}

\begin{remark}
Compared to Section~\ref{subsec:continuous}, we use the score function to approximate $\nabla \log p^N$   rather than $\nabla \log m^N$. This choice is motivated by two main reasons. First, when $b(x) = -x$, the invariant measure $m^*$ is explicitly known and Gaussian. Second, the proof of our discrete-time result relies on key estimates involving $\nabla \log p^N_{T-t}(\overleftarrow{X}_t)$. This dependence also explains why we restrict ourselves to the case $\alpha = 1$, in contrast to Section~\ref{subsec:continuous}: for $\alpha \neq 1$, the definition of $\overleftarrow{X}$ changes, and the required estimate no longer holds.
\end{remark}

%{\color{red} $\alpha=1$ (it seems the argument here only works for $\alpha=1$)}: 

Take a linear test functional 
\[G(m):= \langle \varphi, m\rangle. \]
We consider the error between $G(\mu^h_{T-\delta}) $ and $ G(m_0)$.

\begin{assumption}\label{assum:phi}
    Assume that $\varphi$ is bounded and of bounded derivatives $\nabla \varphi, \Delta \varphi$, and that  $g$ is increasing and convex with $g(\delta)=\delta$.
\end{assumption}

\begin{theorem}\label{thm:discrete}
    Under Assumptions \ref{assum:score_discrete} and \ref{assum:phi}, we have
\begin{multline*}
\mathbb{E}\Big[\left|G\big(\mu^h_{T-\delta}\big) - G(m_0)\right|^2 \Big] \lesssim \\ \big(\|\nabla\varphi\|^2_\infty + \|\Delta\varphi\|^2_\infty\big)\delta^2+ \|\varphi\|^2_\infty \Big(e^{-g(T)} + \frac{1}{N}\Big) +C_T\|\nabla\varphi\|^2_\infty \epsilon^2 +\dot g(T) e^{2(\delta -g(T))}  +
\\ C_T \|\nabla\varphi\|^2_\infty \mathcal{I}(m^N_{\delta}| m^*) h \Big(\sum_{k=0}^{K-1}  \int_{t_k}^{t_{k+1}}|\dot g(T-t)|^2 \d t e^{2(\delta -g(T-t_{k+1}))} \Big).
\end{multline*}
\end{theorem}

In order to well control the discretization error, we are interested in solving the minimization problem:
 \[\inf_{g(\delta)=\delta, g(T)= T'}\sum_{k=0}^{K-1}  \int_{t_k}^{t_{k+1}}|\dot g(T-t)|^2 \d t e^{-2g(T-t_{k+1})}. \]
Approximately, we may solve the following continuous-time optimal control problem. 
\begin{corollary}
\label{cor:ft_opt_pb}
    The minimization problem
    \[\inf_{g(0)=0, g(T)= T'}\int_0^T |\dot g(t)|^2 e^{-2 g(t)}dt  \]
    is solved by the following optimal scheduling:
   \[g^*(t) := - \ln\left(1-t\frac{1-e^{-T'}}{T}\right).\]
\end{corollary}
\begin{remark}
   Note that the minimized functional in Corollary~\ref{cor:ft_opt_pb} differs from the discretization error bound in Theorem~\ref{thm:discrete} by the quantity  
\begin{equation*}
\sum_{k=0}^{K-1}  \int_{t_k}^{t_{k+1}}|\dot g_{T-t}|^2 \left( e^{-2g(T-t_{k+1})} - e^{-2g(t)} \right)\, \mathrm{d}t \\ + \dot g(T) e^{2(\delta - g(T))}.
\end{equation*} 
We first observe that the first term is of order \( O(h) \). Moreover, evaluating the second term at the choice \( g = g^* \) gives  
\[
\dot g^*(T) e^{-2g^*(T)} = \frac{1}{T} \left( e^{-T'} - e^{-2T'} \right) = O(e^{-T'}).
\]  
Consequently, the minimal discretization error bound can differ from the bound obtained with \( g = g^* \) by at most \( O(h) + O(e^{-T'}) \), which is negligible.
\end{remark}

\section{Analog in finite-state generative model setting}

Recently, the finite-state generative model driven by the time-reversal Markov jump process has been earning an increasing interest. In this section, we briefly explain how our method can adapt to this context. Since it is not within the main scope of the paper, we will only focus on the continuous time model (without time discretization error) for an oversimplified one-token model. 

Consider an example where the token can take values from $\mathcal{K}:=\{0, 1\}$. As in the diffusion model, we initialize the forward process at the distribution $m_0\in \R^2 $ on $\mathcal{K}$. Eventually we shall sample according to $m_0$, using a time-reversal process. Let the forward process be a jump process with jump rate $\lambda$ and transition kernel $\tau(i,j) = \frac{1}{2} $ for $i,j\in \mathcal{K}$. In this setting, the marginal distribution $(m_t)_t$ evolves as the solution to the equation:
\[\frac{d}{dt} m_t(i) = \lambda \big( \frac12 -m_t(i) \big), \quad\mbox{for all $i\in \mathcal{K}$}. \]
Therefore $m_t(i) =  e^{-\lambda t} m_0(i) + (1-e^{-\lambda t}) \frac12$, and the marginal $m_t$ exponentially converges to the invariant measure, the uniform distribution on $\mathcal{K}$, denoted by $m^*$. As studied in \cite{discretediffusion}, by denoting the score function $s_t(i) := \frac{m_t(i) - m_t(1-i)}{m_t(i)}$, the time-reversal jump process is of the jump rate $\overleftarrow{\lambda}_t(i):= \frac{\lambda}{2} (2-s_t(i))$ and of the transition kernel $\overleftarrow{\tau}_t(i,1-i):=\frac{\lambda}{2}\frac{1-s_t(i)}{\overleftarrow{\lambda}_t(i)}$ and $\overleftarrow{\tau}_t(i,i):=\frac{\lambda}{2\overleftarrow{\lambda}_t(i)}$. Let the time-reversal jump process be initialized by the distribution $m_T$. One can verify that its marginal distribution is $\mu_t = m_{T-t}$. The dynamics of $(\mu_t)_t$ reads
\begin{equation*}
    \frac{d}{dt}\mu_t(i) = 
     \frac{\lambda}{2}\Big(\big(1-s_{T-t}(1-i)\big)\mu_t(1-i) 
     - \big(1-s_{T-t}(i)\big)\mu_t(i)\Big).
\end{equation*}
In practice, due to limited access to data, one may initialize the forward process by a distribution $m^N_0$ instead of $m_0$. We will denote by $(m^N_t)_t$ the marginal distribution of the associated forward jump process, and by $(\mu^N_t)_t$ that of the time-reversal process. On the other hand, one may use the data to estimate the score function $s_t$ by a parametrized function $s^*$ and try to control the error so that
\begin{equation}\label{eq:scorematching discrete}
    \sum_{i=0,1}\int_0^T |s_t(i) - s^*_t(i)|^2 m_t^N(i)dt <\epsilon^2.
\end{equation}
Denote by $(\mu^*_t)_t$ the marginal distribution of the time-reversal jump process driven by the function $s^*$ and starting from $m^*$ so that $\mu^*_0=m^*$ and 
\begin{equation}\label{eq:backjump}
    \frac{d}{dt}\mu^*_t(i) = 
     \frac{\lambda}{2}\Big(\big(1-s^*_{T-t}(1-i)\big)\mu^*_t(1-i) \\
     - \big(1-s^*_{T-t}(i)\big)\mu^*_t(i)\Big).
\end{equation}
Now, as in our result for the diffusion model, we aim at measuring the error of the generative model through a smooth test function $G:\R^2 \to \R$ by estimating
\(\big|G(m_0) - G(\mu^*_T) \big|\). 

Similar to the diffusion case, we again rely on the function $U(t,m) := G(\mu_{T}^{*, t,m})$ where $\mu_{T}^{*, t,m}$ is the terminal value at time $T$ of the solution to the equation \eqref{eq:backjump}  given the initial condition $\mu_t:=m$. We have $G(m^N_0) = U(T, m^N_0)=U(T, \mu^N_T)$,  $G(\mu^*_T) = U(0, m^*) $, and more importantly that $U(t, x_0, x_1)$ solves the backward Komologrov equation
\begin{equation}\label{eq:jumpPDE}
    \partial_t U + \frac{\lambda}{2}\sum_{i=0,1}\partial_{x_i} U\Big( \big(1-s^*_{T-t}(1-i)\big)x_{1-i}\\
- \big(1-s^*_{T-t}(i)\big)x_{i}\Big)=0.
\end{equation}
Now note that 
\begin{equation*}
    G(m_0) - G(\mu^*_T) 
= G(m_0) -G(m_0^N)\\ + U(T,\mu^N_T) - U(0, \mu^N_0) + U(0,m^N_T) - U(0,m^*).
\end{equation*} 
While it is easy to obtain 
$\E|G(m_0) -G(m_0^N)|\lesssim \frac{1}{\sqrt N}$ and $\quad |U(0,m^N_T) - U(0,m^*)|\lesssim e^{-\lambda T}$, the estimate of the middle term relies on the PDE \eqref{eq:jumpPDE}. Provided that $\partial_{x_i} U$ is bounded, we can verify that 
\begin{multline*}
   |U(T,\mu^N_T) - U(0, \mu^N_0)| = \\
   \frac{\lambda}{2} \Big| \int_0^T \sum_{i=0,1} \partial_{x_i}U(t,\mu^N_t)
   \Big(\big(s_{T-t}(1-i)-s^*_{T-t}(1-i)\big)\mu^N_t(1-i)
- \big(s_{T-t}(i)-s^*_{T-t}(i)\big)\mu^N_t(i)\Big)dt \Big| \lesssim \epsilon T,
\end{multline*}
where the last inequality is due to assumption on the score matching error \eqref{eq:scorematching discrete}. Therefore, the conclusion reads
\[\E\big|G(m_0) - G(\mu^*_T) \big|\lesssim \frac{1}{\sqrt N} + e^{-\lambda T} + \epsilon T.\]

\begin{remark}
The literature on error estimation for discrete-state diffusion models remains relatively sparse.
We note in particular the following several works that contribute to the understanding of discrete space diffusion models.
\cite{Su2025TheorAnalDFM} derive bounds on TV distance between target and generated distibution in terms of statistical error in the framework of discrete flow matching. \cite{Huang2025Complexity} develop the complexity theory for (masked) discrete diffusion providing bounds on the number of discrete score evaluators necessary to achieve the desired TV accuracy. \cite{Wang2025d2ImprovedTechniques} introduce training techniques for reasoning diffusion models that allow for estimating the KL divergence accuracy in an analytically tractable manner. \cite{Wan2025ErrorAnalysisDiscreteFlow} develop a theoretical framework for deriving non-asymptotic KL divergence and TV error bounds for discrete flows with generator matching. \cite{Kumar2025ScaleWiseVAR} show that the visual autoregressive generation is mathematically a discrete diffusion model allowing to export diffusion-related tools to VAR.

The works that are closest to our approach are \cite{discretediffusion, LiCai2025diffusionlanguagemodels}, quantifying
the error via the Kullback--Leibler (KL) divergence. Note, however, that when the state space \(K\) is
large and the token vocabulary is rich, then the empirical or model distribution \(m_0\) is typically
highly sparse. In that regime, the KL divergence is often ill-posed (infinite or numerically
unstable due to zero coordinates) unless ad hoc smoothing is introduced. Measuring generalization
error through a test regular function, as above, avoids this issue. 
\end{remark}

\section{Discussion on the optimal scheduler}

We give in  Corollary \ref{cor:ft_opt_pb} the scheduler $g^*$ to minimizer the discretization error upper bound obtained in Theorem \ref{thm:discrete}. Here we compute a Gaussian example, to illustrate that the scheduler is indeed often optimal. 

Suppose the target distribution is N$(\mu,\sigma^2)$, that is $m_0=\text{N}(\mu,\Sigma)$. Then, the dynamic
$$
\d X_t= - \dot g(t)X_t+\sqrt{\dot g(t)}\d W_t,
$$
will have the distribution as
$$
m_t=\text{N}\bigg(e^{-g(t)}\mu,\frac12+e^{-2g(t)}\Big(\sigma^2-\frac12\Big)\bigg).
$$
So we can compute $\nabla\log m_t(x)$ by
$$
\nabla\log m_t(x)=-\frac{2(x-e^{-g(t)}\mu)}{1+e^{-2g(t)}(2\sigma^2-1)}.
$$
Therefore, assuming zero estimation error of the score function, we get the reverse discrete dynamic:
\begin{equation*}
\d \overleftarrow{X}_t^h=
\dot g(T-t)\bigg(-\overleftarrow{X}_t^h-\frac{2(\overleftarrow{X}_{t_k}^h-e^{-g(T-t_k)}\mu)}{1+e^{-2g(T-t_k)}(2\sigma^2-1)}+2\overleftarrow{X}_{t_k}^h\bigg)\d t +\sqrt{\dot g(T-t)}\d \overleftarrow{W}_t
\end{equation*}
for $t\in [t_k,t_{k+1})$. Since inside each time interval, it is an OU process, we can explicitly write out the distribution of $\overleftarrow{X}_t^h$ given that Law($\overleftarrow{X}_0^h$)=N$(0,\frac12 I_d)$.
% In particular, given $\overleftarrow{X}^h_{t_{k}}$, $\overleftarrow{X}^h_t$ follows an OU process on $t\in[t_k,t_{k+1})$, then we can write 
% \begin{align*}
% \overleftarrow{X}^h_{t}\sim\text{N}\Bigg(\bigg(-\Big(e^{-2g(T-t)}&\Sigma+\frac 12\big(1-e^{-2g(T-t)}\big)I_d\Big)^{-1}\Big(\overleftarrow{X}_{t_k}^h-e^{-g(T-t)}\mu\Big)+2\overleftarrow{X}_{t_k}^h\bigg)\big(1-e^{-g(T-t_k)+g(T-t)}\big)
% \\&+\overleftarrow{X}^h_{t_k}e^{-g(T-t_k)+g(T-t)},\frac{1-e^{-2(g(T-t_k)-g(T-t))}}{2}\Bigg).
% \end{align*}
% Simplifying the expression, we have
% \begin{multline*}
% \label{eq:gaussian iteration}
% \overleftarrow{X}^h_{t}\sim\text{N}\Bigg(\overleftarrow{X}_{t_k}^he^{-g(T-t_k)+g(T-t)}+ \\
% (1-e^{-g(T-t_k)+g(T-t)})\bigg(-\frac{2(\overleftarrow{X}_{t_k}^h-e^{-g(T-t_k)}\mu)}{1+e^{-2g(T-t_k)}(2\sigma^2-1)}+2\overleftarrow{X}_{t_k}^h\bigg),\notag
% \\\frac{1-e^{-2(g(T-t_k)-g(T-t))}}{2}\Bigg).
% \end{multline*}
% Coefficient of $\overleftarrow{X}^h_{t}$:

% \begin{align*}
% &\frac{e^{-g(T-t_k)+g(T-t)}\big(1+e^{-2g(T-t_k)}(2\sigma^2-1)\big)+2\big(1-e^{-g(T-t_k)+g(T-t)}\big)e^{-2g(T-t_k)}(2\sigma^2-1)}{1+e^{-2g(T-t_k)}(2\sigma^2-1)}
% \\
% &=\frac{e^{-g(T-t_k)+g(T-t)}+2e^{-2g(T-t_k)}(2\sigma^2-1)-e^{-3g(T-t_k)+g(T-t)}(2\sigma^2-1)}{1+e^{-2g(T-t_k)}(2\sigma^2-1)}
% \end{align*}

In order to see the discretization error upper bound obtained in Theorem \ref{thm:discrete} c
an be sharp, we take  $\sigma^2=\frac12$. In this case,  we have
% \begin{multline*}
% \overleftarrow{X}^h_{t}\sim\text{N}\Bigg(\overleftarrow{X}_{t_k}^he^{-g(T-t_k)+g(T-t)} \\+2\big(e^{-g(T-t_k)}-e^{-2g(T-t_k)+g(T-t)}\big)\mu,\\
% \frac{1-e^{-2(g(T-t_k)-g(T-t))}}{2}\Bigg).
% \end{multline*}
% Therefore, we have
% \begin{multline*}
% \E\big[\overleftarrow{X}_{t_{k+1}}^h\big]= \\\E\big[\overleftarrow{X}_{t_{k}}^h\big]e^{-g(T-t_k)+g(T-t_{k+1})} \\+2\big(e^{-g(T-t_k)}-e^{-2g(T-t_k)+g(T-t_{k+1})}\big)\mu,
% \end{multline*}
% which leads to

\begin{multline}
\label{eq:gaussian_iteration} 
\E\big[\overleftarrow{X}_{t_{k}}^h\big]
=\E\big[\overleftarrow{X}_{0}^h\big]e^{-g(T)+g(T-t_{k})} +2\mu\sum_{\ell=1}^ke^{-g(T-t_\ell)+g(T-t_{k})}
\Big(e^{-g(T-t_{\ell-1})}-e^{-2g(T-t_{\ell-1})+g(T-t_{\ell})}\Big)
\\
=2\mu e^{g(T-t_k)} \bigg(-\sum_{\ell=1}^ke^{-2g(T-t_{\ell-1})}+\sum_{\ell=1}^ke^{-g(T-t_{\ell-1})-g(T-t_\ell)}\bigg)
\\
\approx 2\mu e^{g(T-t_k)}\sum_{\ell=1}^k e^{-2g(T-t_{\ell-1})}\dot g(T-t_{\ell-1}) h.
\end{multline}
Note that $\int_0^T e^{-2g(t)}2\dot g(t)dt =1-e^{-2g(T)} \approx 1 $. 
Therefore, the discretization error reads
\begin{multline*}
\mu- \E\big[\overleftarrow{X}_{T}^h\big]
    \approx \mu\sum_{\ell=1}^K\int_{t_{\ell-1}}^{t_\ell} \big( e^{-2g(T-t)} - e^{-2g(T-t_{\ell-1})} \big)2\dot g(T-t_{\ell-1})dt \\
\approx 4\mu h^2\sum_{\ell=1}^K  e^{-2g(T-t_{\ell-1})}|\dot g(T-t_{\ell-1})|^2, 
\end{multline*}
where the right hand side is  approximately the value we aim to minimize in Corollary \ref{cor:ft_opt_pb}. In this sense, we claim  the  scheduler in Corollary \ref{cor:ft_opt_pb} can be optimal. 

In the following we numerically evaluate the expectation $\E\big[\overleftarrow{X}_{T}^h\big]$ by \eqref{eq:gaussian_iteration} according to several common schedulers. Let $T=1$, we shall try:
\begin{itemize}
    \item  The linear scheduler $g(t) = T' t$;
    \item  The optimal scheduler prescribed in Corollary \ref{cor:ft_opt_pb} $g(t) = -\ln (1-t(1-e^{-T'}))$;
    \item The commonly used cosine scheduler $g(t) = -\ln(\cos(\frac{\pi}{2}at))$ where $a$ is such that $\cos(\frac{\pi}{2}a) =e^{-T'}$.
\end{itemize}
We  compare the errors on the computed expectation values for different values of $\sigma$.

% \begin{figure*}[htbp]
% \centering
%   % Top Row
%   \begin{subcaptionbox}{$\sigma=0.1$\label{fig:graph1}}[0.23\textwidth]
%     {\includegraphics[width=\linewidth]{sigma 01.png}}
%   \end{subcaptionbox}
%   \hfill
%   \begin{subcaptionbox}{$\sigma=0.5$\label{fig:graph2}}[0.23\textwidth]
%     {\includegraphics[width=\linewidth]{sigma 05.png}}
%   \end{subcaptionbox}
%   \hfill
%   \begin{subcaptionbox}{$\sigma=1$\label{fig:graph3}}[0.23\textwidth]
%     {\includegraphics[width=\linewidth]{sigma 1.png}}
%   \end{subcaptionbox}
%   \hfill
%   \begin{subcaptionbox}{$\sigma=2$\label{fig:graph4}}[0.23\textwidth]
%     {\includegraphics[width=\linewidth]{sigma 2.png}}
%   \end{subcaptionbox}
%   \caption{Error on expectation with different variances and schedulers}
%   \label{fig:error}
% \end{figure*}

\begin{figure*}[htbp]
\centering
  % Top Row
  \subfloat
  {\includegraphics[width=0.23\textwidth]{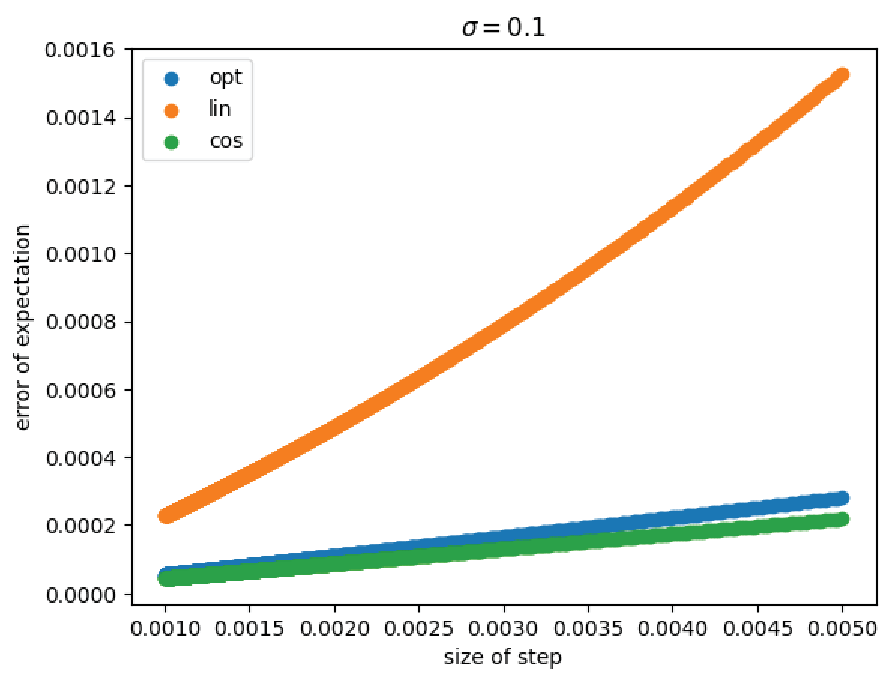}}
  \hfill
  \subfloat
  {\includegraphics[width=0.23\textwidth]{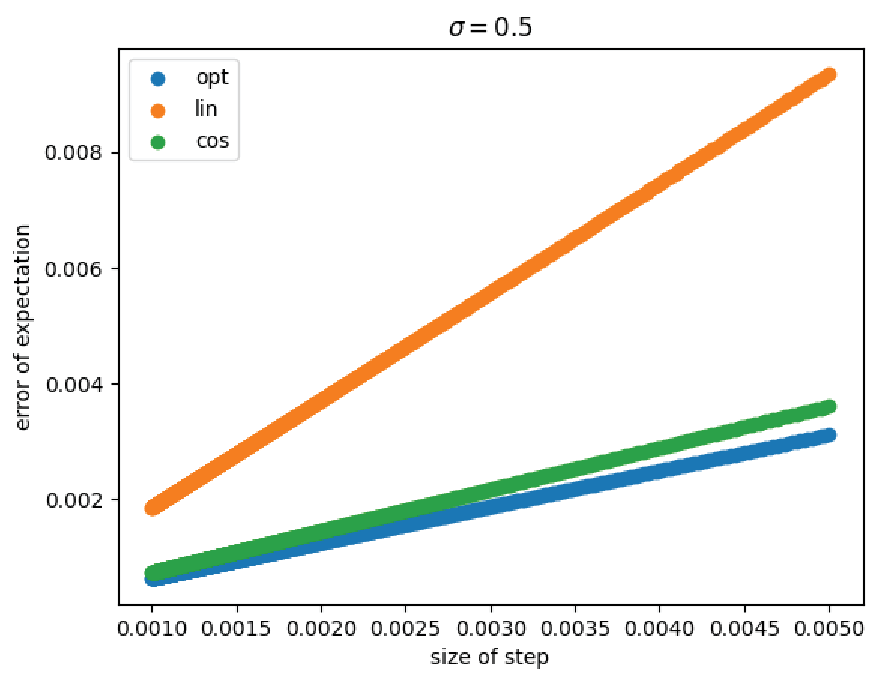}}
  \hfill
  \subfloat
  {\includegraphics[width=0.23\textwidth]{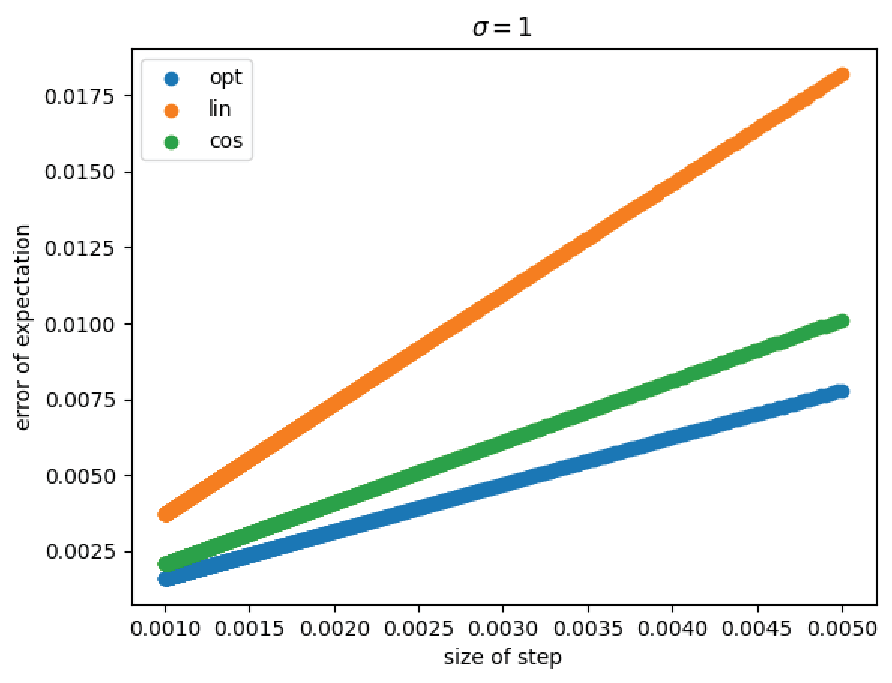}}
  \hfill
  \subfloat
  {\includegraphics[width=0.23\textwidth]{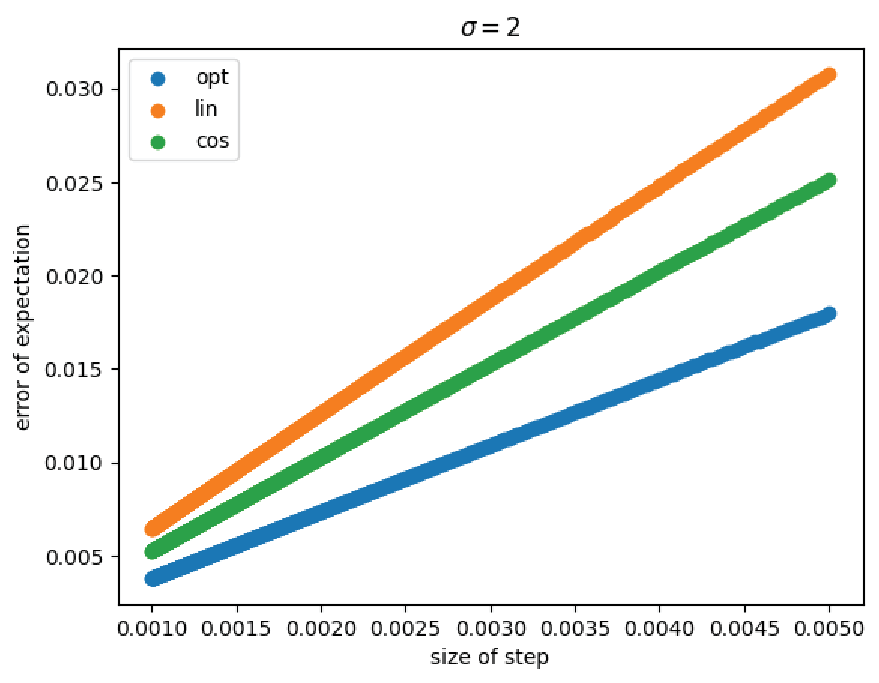}}
  \caption{Error on expectation with different variances and schedulers}
  \label{fig:error}
\end{figure*}

% \begin{figure*}[htbp]
% \centering
%   % Top Row
%   \begin{subcaptionbox}{$\sigma=0.1$\label{fig:graph1}}[0.4\textwidth]
%     {\includegraphics[width=\linewidth]{sigma 01.png}}
%   \end{subcaptionbox}
%   \hfill
%   \begin{subcaptionbox}{$\sigma=0.5$\label{fig:graph2}}[0.4\textwidth]
%     {\includegraphics[width=\linewidth]{sigma 05.png}}
%   \end{subcaptionbox}

%   \vspace{0.4cm} % vertical spacing between rows

%   % Bottom Row
%   \begin{subcaptionbox}{$\sigma=1$\label{fig:graph3}}[0.4\textwidth]
%     {\includegraphics[width=\linewidth]{sigma 1.png}}
%   \end{subcaptionbox}
%   \hfill
%   \begin{subcaptionbox}{$\sigma=2$\label{fig:graph4}}[0.4\textwidth]
%     {\includegraphics[width=\linewidth]{sigma 2.png}}
%   \end{subcaptionbox}

%   \caption{Error on expectation with different variances and schedulers}
%   \label{fig:error}
% \end{figure*}
As shown in Figure \ref{fig:error}, the prescribed optimal scheduler stays optimal or almost optimal in different cases.

\section{Related works}

\paragraph{State-of-the-art results} There exists an extensive literature on the convergence of denoising diffusion models. While early works such as \cite{debortoli2021diffschrodingerbridge,Pidstrigach2022-se,Block2020-ty,debortoli2023convergencediffusion} and more recent works such as \cite{Choi2025-zt} proved convergence results with exponential dependencies in the problem parameters, those results have been refined in subsequent analyses. 

Early works proved error bounds with polynomial dependencies with respect to the dimension or linear dependency with respect to the dimension but under strong assumptions on the regularity of the score function, see \cite{Li2023-pn,Li2024-vu,lee2023generaldistributions,chen2023diffmodels,Gupta2024-ob}. Those results were also extended to bound the probability flow ODE associated with diffusion models \cite{Chen2023-nj,Benton2023-mh}. 

The first bounds on the Kullback-Leibler divergence \emph{linear} with respect to the dimension were proven in \cite{benton2024nearlydlinearconvergencebounds,conforti2025diffusion} and do not require any smoothness condition on the score. 
The dependency with respect to $\varepsilon$ was mitigated in \cite{Li2024-em,Li2024-rn} to $O(d/\varepsilon)$.  To the best of our knowledge the best results obtained so far are of complexity $O(\min(d, d^{2/3}L^{1/3}, d^{1/3}L) \varepsilon^{-2/3})$ in \cite{Jiao2024-si} leveraging a relaxed Lipschitz constant. It is also possible to improve those results by assuming that the data distribution is low dimensional \cite{potaptchik2025linearconvergencediffusionmodels,Liang2025-nn}. 
To the best of our knowledge, the result most closely related to our contribution is \cite{li2025dimensionfreeconvergencediffusionmodels} which identify dimension-free convergence rates for diffusion models in the case of Gaussian Mixture Models. 

Those state-of-the-art results can even be improved with a dependency of $O(d^{1+2/K}\varepsilon^{-1/K})$ upon using higher order samplers \cite{Li2025-wm} and $O(d^{5/4}\varepsilon^{1/2})$ using accelerated samplers, see also \cite{Pedrotti2023-zj}.  We also highlight that it is possible to improve those convergence rates even further, even deriving sub-linear time complexity results, by considering parallel sampling strategies as in \cite{Gupta2024-ob,debortoli2025accelerateddiffusionmodelsspeculative}.
Those results have been partially extended to the case of stochastic interpolants \cite{albergo2023stochasticinterpolants} in \cite{liu2025finitetimeanalysisdiscretetimestochastic} (with sub-optimal rates). 

\paragraph{Convergence in other modes.} The convergence of the diffusion models regarding other metrics has been investigated in several papers. For instance in \cite{mbacke2024noteconvergencedenoisingdiffusion,Cheng_2024,Gao2023-pc,Gao2025-ni,Yu2025-fz} show the convergence of diffusion models with respect to the Wasserstein distance of order $2$. Contrary to \cite{debortoli2023convergencediffusion} their results is not exponential in some parameters of the diffusion models. Other results on the convergence of diffusion models in Wasserstein distance including \cite{lee2022polynomialcomplexity} which improves significantly improves upon \cite{debortoli2023convergencediffusion}. However it relies on strong assumptions regarding the Lipschitz constant of the denoiser. While \cite{Reeves2025-fk} draws connection with information theory their analysis does not provide insights with respect to the discretization of the reverse process and its convergence.  

\paragraph{Other theoretical analyses.} Finally, we highlight a few other areas of theoretical exploration of diffusion models. \cite{LiCai2025diffusionlanguagemodels} explored the convergence of discrete categorical diffusion models, while \cite{Biroli2024-wg} characterized different modes of the backward diffusion leveraging tools from statistical physics.  \cite{chen2023scoreapproximationestimationdistribution} provides sample complexity bounds for distribution estimation in low dimensional settings. Finally, \cite{buchanan2025edgememorizationdiffusionmodels} identified memorization and generalization conditions for diffusion models. \cite{Chen2024-ru,Gatmiry2024-ta} focus on the \emph{learning} problem in diffusion models and provide result for general Gaussian mixtures. Combining learning guarantees with the aforementioned convergence results several works have derived near minimax optimality results \cite{Cai2025-ww,Oko2023-uv,azangulov2025convergencediffusionmodelsmanifold}.

% Check that we cite all Gen Li papers.  

\section{Proof of Theorem \ref{thm:continuous}}\label{proofofthm:cont}

Introduce 
\[U(t, m): = G\big({\rm Law}(\overleftarrow{X}^*_{T-\delta}|\overleftarrow{X}^*_t\sim m)\big).\]
Note that by the flow property we have
\[U(T - \delta, \mu^*_{T-\delta}) = U(0, \mu^*_0),\]
and on the other hand, the function $U$ satisfies  the backward Kolmogorov equation:
\[\partial_t U + \int\left( D_m U \big(-b + \frac{\alpha^2 +1}{2}S^*_{T-\cdot}\big) + \frac{\alpha^2}{2}\nabla\cdot D_m U \right) d m = 0.\]

\begin{lemma}[Regularity bound for $U$]\label{lem:regularity}
    Under Assumption  \ref{assum:S} (1) and \ref{assum:G}, we have the following regularity estimate on $U$:
    \begin{equation*}
        \begin{aligned}
            \big| D_m U(t, m, x) \big| \le C_T \|D_m G\|_\infty, \quad \Big| \frac{\delta U}{\delta m}(t, m, x) \Big| \le C \Big\|\frac{\delta G}{\delta m} \Big\|, \quad
            \Big| \frac{\delta^2 U}{\delta m^2}(t, m, x) \Big| \le C \Big\|\frac{\delta^2 G}{\delta m^2} \Big\|.
        \end{aligned}
    \end{equation*}
\end{lemma}

\begin{lemma}\label{lem:UN}
Assume that $\varphi$ has
bounded derivatives $\frac{\delta \varphi}{\delta m}$ and $\frac{\delta^2 \varphi}{\delta m^2}$. Then, there exists a constant $C$, depending on the bounds on the first and second-order derivatives, such that, 
\begin{equation*}
    \E^{(x^i_0)_i\sim m_0^{\otimes N}}\Big[\big|\varphi(m_T)-\varphi(m_T^N)\big|^2\Big]\leq \frac{C}{N}.
\end{equation*}
\end{lemma}
The proof of Lemma \ref{lem:regularity} and \ref{lem:UN} is postponed to the Appendix.

\begin{proof}[Proof of Theorem \ref{thm:continuous}]
    The weak error between  $\mu^*_{T-\delta}$ and $\mu^N_{T-\delta} =m^N_\delta$ reads
\begin{equation*}
\begin{aligned}
        & G\big(\mu^*_{T-\delta}\big) - G\big(m^N_\delta\big)\\
     =&~~  U\big(T-\delta, \mu^*_{T-\delta}\big) - U\big(T-\delta, \mu^N_{T-\delta}\big)\\
    =&~~  U\big(0, m^*\big) - U\big(0, m^N_T\big)\\
    & \quad\quad\quad - \int_0^{T-\delta} 
    \int \left(\partial_t U + D_m U \big(-b +\frac{\alpha^2 +1}{2} \nabla\log m^N_{T-t}\big) + \frac{\alpha^2}{2}\nabla\cdot D_m U \right)\big(t, m^N_{T-t},  x\big)  m^N_{T-t}(dx)  dt\\
    =&~~ U\big(0, m^*\big) - U\big(0, m^N_T\big)\\
    & \quad\quad\quad - \frac{\alpha^2+1}{2}\int_0^{T-\delta} 
    \int \left( D_m U ( \nabla\log m^N_{T-t} - S^*_{T-t})\right)\big(t, m^N_{T-t},  x\big)  m^N_{T-t}(dx)  dt.
\end{aligned}
\end{equation*}
Therefore, by Pinsker's inequality and Assumption \ref{assum:contraction}, we have 
\begin{multline*}
    \mathbb{E} \left[ \left| G\big(\mu^*_{T-\delta}\big) - G\big(m^N_\delta\big) \right|^2 \right] \\  
\le C \mathbb{E} \left[ \left|U\big(0, m^*\big) - U\big(0, m^N_T\big) \right|^2 \right] + C_T \epsilon^2(\alpha^2 +1)^2 \|D_m U\|^2_\infty \\
\le C \left\|\frac{\delta U}{\delta m}\right\|^2_\infty H(m_T|m^*) + C \mathbb{E} \left[ \left|U\big(0, m_T\big) - U\big(0, m^N_T\big) \right|^2 \right] + C_T\|D_m U\|^2_\infty \epsilon^2(\alpha^2 +1)^2 \\
\le  C \left\|\frac{\delta U}{\delta m}\right\|^2_\infty e^{-2\rho T} H(m_0| m^*)  + C \mathbb{E} \left[ \left|U\big(0, m_T\big) - U\big(0, m^N_T\big) \right|^2 \right]+  C_T\|D_m U\|^2_\infty \epsilon^2 (\alpha^2 +1)^2,
\end{multline*} 
where constant $C$ may vary from line to line. 
Moreover, due to the boundedness of the linear derivative of $G$ and $U(0,\cdot)$ (Lemma \ref{lem:regularity}), by It\^o's formula and Lemma \ref{lem:UN} we obtain:
\begin{equation*}
    \begin{aligned}
        &\mathbb{E}\left[\left| G\big(m_0\big) - G\big(m_\delta\big)  \right|^2\right] 
        = \mathbb{E}\left[\left| \int_0^{\delta } \int_{\mathbb{R}^d}\Big(D_m G(m_t, x)b(x) + \frac12\nabla \cdot D_m G(m_t,x) \Big) m_t(dx)dt \right|^2\right]  \le C\delta^2,\\
        &\mathbb{E}\left[\left| G\big(m_\delta\big) - G\big(m^N_\delta\big)  \right|^2\right] \le \frac{C}{N}, 
\quad\quad \mathbb{E}\left[\left| U\big(0,m_T\big) -  U\big(0,m^N_T\big)\right|^2\right] \le \frac{C_T}{N}. 
    \end{aligned}
\end{equation*}
Finally, the desired estimate follows.
% where $C_t$ depends on the bound of the linear derivatives of $G, ~U$ and the second moment of $m_t$, but not on the dimension. See Lemma 5.10 in \cite{10.1214/19-EJP298}.
\end{proof}

\section{Proof of Theorem \ref{thm:discrete} }\label{sec:proofofthmdiscrete}

Let us define a family of functions $\{u^{(k)}(\cdot;x)\}_{\{k=0,1,2,..,K-1,x\in\R\}}$ and $\{v^{(k)}\}_{k=0,1,..,K}$ by
\begin{equation}\label{ukfeymankac}
    \text{For }t\in [t_k, t_{k+1}),\quad u^{(k)}(t,x;x_k)\coloneqq\mathbb{E}\Big[v^{(k+1)}(\overleftarrow{X}^h_{t_{k+1}})\Big|\overleftarrow{X}^h_t = x, \overleftarrow{X}^h_{t_k} =x_k \Big], 
\end{equation}
with
$$
v^{(K)}(x)\coloneqq\varphi(x),\text{ and }v^{(k)}(x)\coloneqq u^{(k)}(t_{k},x;x)\text{, for }k=0,1,2,..,K-1.
$$
% Use the explicit formula of OU process:
% $$
% X_t=X_0e^{at}+\frac ba(1-e^{at})+\sigma e^{at}\int_0^t e^{-as}\d W_s,
% $$
% for 
% $$
% \d X_t=(aX_t+b)\d t+\sigma \d W_t.$$
% Then,

\begin{lemma}\label{lem:estimate uk}
    Given the definition of $\{v^{(k)}\}_{k=0,1,..,K}$ and $\{u^{(k)}\}_{k=0,1,2,..,K-1}$ above, fix $k\in\{0,1,..,K-1\}$, for all $x\in\R^d$, and $t\in[t_k,t_{k+1})$ we have
    $$
    \|u^{(k)}(t,\cdot;x)\|_\infty\leq \|\varphi\|_\infty,\quad\|\nabla u^{(k)}(t,\cdot;x)\|_\infty\leq \|\nabla\varphi\|_\infty e^{(1+\|S^*\|_{\text{Lip},1,\infty})g(T)}.
    $$
    Moreover, for all $k\in\{0,1,..,K\}$
    $$
    \|v^{(k)}\|_\infty\leq \|\varphi\|_\infty,\quad\|\nabla v^{(k)}\|_\infty\leq \|\nabla\varphi\|_\infty e^{(1+\|S^*\|_{\text{Lip},1,\infty})T}.
    $$
\end{lemma}

\noindent The proof is postponed to the Appendix.

% , define
% \[u(t, x; t_k, x_k):= \mathbb{E}\Big[\varphi(\overleftarrow{X}^h_{T-\delta})\Big|\overleftarrow{X}^h_t = x, \overleftarrow{X}^h_{t_k} =x_k \Big].\]
% \JZ{
% $$U(t,\mu;t_k,x_k)=\E\bigg[G\Big(\text{Law}\big(\overleftarrow{X}^h_T\big)\Big)\bigg|\text{Law}\big(\overleftarrow{X}^h_t\big)=\mu,\overleftarrow{X}^h_{t_k}=x_k\bigg].$$
% $$
% U(t,\mu;t_k,x_k)=G\big(\mu_{T-t}^{*,x_k,t_k}(\mu)\big),
% $$
% $$
% U(t,\mu)=\E_{x\sim \mu}\Big[G\big(\mu_{T-t}^{*,x,t}(\mu)\big)\Big]\text{ or }U(t,\mu)=G\Big(\E_{x\sim \mu}\big[\mu_{T-t}^{*,x,t}(\mu)\big]\Big)
% $$}
% \[U(t, t_k, \mu)=G\Big(\E_{x_k\sim \mu\circ X_k^{-1}}\big[\mu_{T-t}^{*,x_k,t}(\mu)\big]\Big)\]
Further, on $ [t_k, t_{k+1})$ the function $u^{(k)}(\cdot;x_k)$ solves the PDE:
\begin{equation}\label{eq:PDE piecewise}
    \begin{cases}\partial_t u^{(k)}(t,x;x_k) + \frac12 \dot g(T-t) \Delta u^{(k)}(t,x;x_k) +  \dot g(T-t) \Big(x  +  S^*_{T-{t_k}}(x_{k})\Big) \nabla u^{(k)}(t,x;x_k) = 0.
    \\
    u^{(k)}(t_{k+1},x;x_k)=v^{(k+1)}(x).
    \end{cases}
\end{equation}
% Meanwhile, define 
% \[u(t_k,x):= u(t_k,x;t_k,x)=\mathbb{E}\Big[\varphi(\overleftarrow{X}^h_{T-\delta})\Big|\overleftarrow{X}^h_{t_k} =x \Big],\text{ for }k=0,1,..,K.\]
Since 
\[v^{(k)}(x)=u^{(k)}(t_k,x;x) = \mathbb{E}\Big[v^{(k+1)}( \overleftarrow{X}^h_{t_{k+1}})\Big|\overleftarrow{X}^h_{t_k} =x\Big]=\E\Big[\varphi(\overleftarrow{X}^h_{T-\delta})\Big|\overleftarrow{X}^h_{t_k} =x\Big],\]
we get $v^{(0)}(x)=\mathbb{E}\big[\varphi(\overleftarrow{X}^h_{T-\delta})\big|\overleftarrow{X}^h_{0} =x \big]$, and $\varphi(x)=v^{(K)}(x)$. Therefore, we have 
\begin{equation*}
G(\mu^h_{T-\delta}) - G(m^N_\delta) = \mathbb{E}\Big[\varphi(\overleftarrow{X}^h_{T-\delta})\Big] - \langle \varphi, m_\delta^N\rangle 
    =\langle v^{(0)}, m^*\rangle - \langle v^{(K)}, m_\delta^N\rangle.
\end{equation*}
Now we are ready to prove the main result on the discrete-time model.

\begin{proof}[Proof of Theorem \ref{thm:discrete}]
Note that 
\[G(\mu^h_{T-\delta}) - G(m^N_\delta) =
\langle v^{(0)}, m^*\rangle -\langle v^{(0)}, m^N_T\rangle + \langle v^{(0)}, m^N_T\rangle - \langle v^{(K)}, m_\delta^N\rangle.\]
By Pinsker's inequality and Lemma \ref{lem:UN}, the difference  $\langle v^{(0)}, m^*\rangle-\langle v^{(0)}, m^N_T\rangle$ is controlled as follows:
\begin{align}\label{diff:mstarmT}
&\E^{(x^i_0)_i\sim m_0^{\otimes N}}\left[\Big|\langle v^{(0)}, m^*\rangle-\langle v^{(0)}, m^N_T\rangle\Big|^2\right]\\
&\lesssim\Big|\langle v^{(0)}, m^*\rangle-\langle v^{(0)}, m_T\rangle\Big|^2 + \E^{(x^i_0)_i\sim m_0^{\otimes N}}\left[\Big|\langle v^{(0)}, m_T\rangle-\langle v^{(0)}, m^N_T\rangle\Big|^2\right]\notag\\
&\leq \|v^{(0)}\|^2_\infty \text{TV}(m^*,m_T)^2 + \E^{(x^i_0)_i\sim m_0^{\otimes N}}\left[\Big|\langle v^{(0)}, m_T\rangle-\langle v^{(0)}, m^N_T\rangle\Big|^2\right]\notag\\
&\lesssim  \|v^{(0)}\|^2_\infty \Big(H(m_T|m^*) + \frac{1}{N}\Big) ~\lesssim~ \|\varphi\|^2_\infty \Big(e^{-2g(T)} + \frac{1}{N}\Big) .
\end{align}
To estimate the difference $\langle v^{(K)},m_\delta^N\rangle - \langle v^{(0)}, m^N_T\rangle$, it follows from  It\^o's formula  and the PDEs \eqref{eq:PDE piecewise} that 
\begin{equation*}
\begin{aligned}
\langle v^{(K)}&,m_\delta^N\rangle - \langle v^{(0)}, m^N_T\rangle=\E\Big[v^{(K)}\big(\overleftarrow{X}_{T-\delta}\big)-v^{(0)}\big(\overleftarrow{X}_{0}\big)\Big]=\sum_{k=0}^{K-1}\E\Big[v^{(k+1)}\big(\overleftarrow{X}_{t_{k+1}}\big)-v^{(k)}\big(\overleftarrow{X}_{t_k}\big)\Big]
\\
=&\sum_{k=0}^{K-1}\E\Big[u^{(k)}\big(t_{k+1},\overleftarrow{X}_{t_{k+1}};\overleftarrow{X}_{t_k}\big)-u^{(k)}\big(t_k,\overleftarrow{X}_{t_k};\overleftarrow{X}_{t_k}\big)\Big]
\\
\overset{\text{It\^o's}}{=}& \sum_{k=0}^{K-1} \mathbb{E}\bigg[\int_{t_k}^{t_{k+1}} \bigg(\partial_t u^{(k)}\big(t,\overleftarrow{X}_{t};\overleftarrow{X}_{t_k}\big)+\dot g(T-t))\Delta u^{(k)}\big(t,\overleftarrow{X}_{t};\overleftarrow{X}_{t_k}\big)
\\
&\qquad\qquad\qquad\qquad+\dot g(T-t) \Big(- \overleftarrow{X}_{t}+ \nabla \ln p^N_{T-t}\big(\overleftarrow{X}_{t}\big)\Big)  \nabla u^{(k)}\big(t,\overleftarrow{X}_{t};\overleftarrow{X}_{t_k}\big) \bigg)\d t\bigg]
\\
\overset{\eqref{eq:PDE piecewise}}{=}&\sum_{k=0}^{K-1} \mathbb{E}\bigg[\int_{t_k}^{t_{k+1}}\dot g(T-t) \Big( \nabla \ln p^N_{T-t}(\overleftarrow{X}_{t}) - S^*_{T-t_k}(\overleftarrow{X}_{t_k}) \Big) \nabla u^{(k)}\big(t,\overleftarrow{X}_{t};\overleftarrow{X}_{t_k}\big) \d t\bigg]
\end{aligned}
\end{equation*}
where $\overleftarrow{X}_0\sim m^N_T$. Therefore,
\begin{equation*}
    \begin{aligned}
        \Big(\langle v^{(K)}&,m_\delta^N\rangle - \langle v^{(0)}, m^N_T\rangle\Big)^2
        \\
        \leq &\;\E\Bigg[\bigg\{\sum_{k=0}^{K-1}\int_{t_k}^{t_{k+1}}\dot g(T-t) \Big( \nabla \ln p^N_{T-t}(\overleftarrow{X}_{t}) - S^*_{T-t_k}(\overleftarrow{X}_{t_k}) \Big) \nabla u^{(k)}\big(t,\overleftarrow{X}_{t};\overleftarrow{X}_{t_k}\big) \d t\bigg\}^2\Bigg]
        \\
        \leq &\;\E\Bigg[\bigg\{\sum_{k=0}^{K-1}\int_{t_k}^{t_{k+1}}\dot g(T-t) \Big\| \nabla \ln p^N_{T-t}(\overleftarrow{X}_{t}) - S^*_{T-t_k}(\overleftarrow{X}_{t_k}) \Big\|^2\d t\bigg\}
        \\
        &\qquad\qquad\cdot  \bigg\{\sum_{k=0}^{K-1}\int_{t_k}^{t_{k+1}}\dot g(T-t) \big\|\nabla u^{(k)}\big(t,\overleftarrow{X}_{t};\overleftarrow{X}_{t_k}\big)\big\|^2 \d t\bigg\}\Bigg].
    \end{aligned}
\end{equation*}

By \eqref{nablaubdd}, we have
\begin{equation}\label{eq:nabla u bound}
\begin{aligned}
\sum_{k=0}^{K-1}\int_{t_k}^{t_{k+1}}\dot g(T-t) \big\|\nabla u^{(k)}\big(t,\overleftarrow{X}_{t};\overleftarrow{X}_{t_k}\big)\big\|^2 \d t&\leq \big\|\nabla u^{(k)}\big\|_\infty^2\int_{t_0}^{t_{K}}\dot g(T-t)\d t
\\
&\leq g(T)\|\nabla\varphi\|^2_\infty e^{(2+2\|S^*\|_{\text{Lip}})g(T)} .
\end{aligned}
\end{equation}

In order to deal with the first term, using the assumption \eqref{eq:dicrete score} and the estimate (38) in \cite{conforti2025diffusion}, we get
\begin{equation*}
\begin{aligned}
    & \mathbb{E}\Big[\sum_{k=0}^{K-1}\int_{t_k}^{t_{k+1}}  \dot g(T-t)\Big\|\nabla \ln p^N_{T-t}(\overleftarrow{X}_{t}) - S^*_{T-{t_k}}(\overleftarrow{X}_{t_k})\Big\|^2\d t\Big]\\
    \le ~ & 2\sum_{k=0}^{K-1} \int_{t_k}^{t_{k+1}} \dot g(T-t)\d t\cdot \mathbb{E}\bigg[\Big\|\nabla \ln p^N_{T-t_k}(\overleftarrow{X}_{t_k}) - S^*_{T-t_k}(\overleftarrow{X}_{t_k})\Big\|^2\bigg]\\
    & \quad\quad\quad +2 \sum_{k=0}^{K-1} 
    \mathbb{E}\Big[\int_{t_k}^{t_{k+1}} \dot  g(T-t) \big|\nabla \ln p^N_{T-t}(\overleftarrow{X}_{t}) - \nabla \ln p^N_{T-t_k}(\overleftarrow{X}_{t_k})\big|^2\d t\Big]\\
    \le ~& 2\epsilon^2 + C \sum_{k=0}^{K-1} \int_{t_k}^{t_{k+1}}\dot g(T-t) \big(\mathcal{I}(m^N_{T- t}| m^*) - \mathcal{I}(m^N_{T- t_k}| m^*)  \big)\d t,
\end{aligned}
\end{equation*}
% for some constant $C$. Note that
% \begin{equation*}
% \begin{aligned}
%     &\sum_{k=0}^{K-1} \dot g_{K-k} \big(\mathcal{I}(m^N_{T- t_{k+1}}| m^*) - \mathcal{I}(m^N_{T- t_{k}}| m^*)  \big) \\
%     =~&\sum_{k=0}^{K-1} \big(\dot g_{K-k+1} - \dot g_{K-k} \big) \mathcal{I}(m^N_{T- t_{k}}| m^*) +\dot g_K \mathcal{I}(m^N_{T- t_{1}}| m^*) -\dot g_1  \mathcal{I}(m^N_{\delta}| m^*)
% \end{aligned}
% \end{equation*} 
Moreover, as proved in Proposition 3  of \cite{conforti2025diffusion}, we have 
\[
 \mathcal{I}(m^N_{T-t}| m^*)\leq e^{2 (g(T-t)-g(T-s))}\mathcal{I}(m^N_{T-s}| m^*)\le \mathcal{I}(m^N_{T-s}| m^*) \quad\mbox{for $t<s$}. \]
% {\color{red}and
% \[
% \mathcal{I}(m^N_{\delta}| m^*)\leq e^{}
% \]}
Therefore, 
\begin{align*}
&\sum_{k=0}^{K-1} \int_{t_k}^{t_{k+1}}\dot g(T-t) \big(\mathcal{I}(m^N_{T- t}| m^*) - \mathcal{I}(m^N_{T- t_k}| m^*)  \big)\d t
\\
\le &  \sum_{k=0}^{K-1} \big( g(T-t_{k}) -  g(T-t_{k+1} )\big)\big(\mathcal{I}(m^N_{T- t_{k+1}}| m^*) - \mathcal{I}(m^N_{T- t_k}| m^*)  \big)\\
 \le &  \sum_{k=1}^{K-1}\big( g(T-t_{k+1}) -  2g(T-t_{k} ) + g(T-t_{k-1})\big) \mathcal{I}(m^N_{T- t_k}| m^*)
+ \big( g(T-t_{K-1}) -  g(\delta)\big)\mathcal{I}(m^N_{\delta}| m^*)\\
\le &   \mathcal{I}(m^N_{\delta}| m^*) \left(\sum_{k=1}^{K-1}\big( g(T-t_{k+1}) -  2g(T-t_{k} ) + g(T-t_{k-1})\big) e^{2(\delta -g(T-t_k))}
+  g(T-t_{K-1}) -  \delta \right)\\
=&  \mathcal{I}(m^N_{\delta}| m^*) \left(\sum_{k=0}^{K-1}  \big( g(T-t_{k}) -  g(T-t_{k+1} )\big)\big( e^{2(\delta -g(T-t_{k+1}))}- e^{2(\delta -g(T-t_k))}\big) + \big( g(T) -  g(T-h )\big) e^{2(\delta -g(T))}\right)\\
\le & \mathcal{I}(m^N_{\delta}| m^*) \left(\sum_{k=0}^{K-1} 2 \big( g(T-t_{k}) -  g(T-t_{k+1} )\big)^2 e^{2(\delta -g(T-t_{k+1}))} + h\dot g(T) e^{2(\delta -g(T))}\right)\\
\leq & h \mathcal{I}(m^N_{\delta}| m^*)\left(\sum_{k=0}^{K-1} 2 \int_{t_k}^{t_k+1}|\dot g(T-t)|^2 \d t e^{2(\delta -g(T-t_{k+1}))} + \dot g(T) e^{2(\delta -g(T))}\right)
\end{align*}
% Assuming that $k\mapsto \dot g_k$ is increasing, we obtain
% \begin{equation*}
% \begin{aligned}
%     \sum_{k=0}^{K-1} \dot g_{K-k}& \big(\mathcal{I}(m^N_{T- t_{k+1}})| m^*) - \mathcal{I}(m^N_{T- t_k}| m^*)  \big) \\
%     \le ~&  \mathcal{I}(m^N_{\delta}| m^*)\left(\sum_{k=0}^{K-1} \big(\dot g_{K-k+1} - \dot g_{K-k} \big) e^{-2 g_{K-k}} +\dot g_K e^{-2 g_{K-1}} -\dot g_1 \right)\\
%     = ~& \mathcal{I}(m^N_{\delta}| m^*)\left(\sum_{k=0}^{K-1}\dot g_{K-k}\big(e^{-2 g_{K-k-1}} -e^{-2 g_{K-k}}  \big)\right)\\
%     \le~& \mathcal{I}(m^N_{\delta}| m^*)\sum_{k=1}^{K} |\dot g_{k}|^2 h e^{-2 g_{k-1}}.
% \end{aligned}
% \end{equation*} 
Together with \eqref{eq:nabla u bound}, we finally have
\begin{equation*}
    \begin{aligned}
        &\Big(\langle v^{(K)},m_\delta^N\rangle - \langle v^{(0)}, m^N_T\rangle\Big)^2 \\
\le&   C_T \left(\epsilon^2+ h \mathcal{I}(m^N_{\delta}| m^*)\Big(\sum_{k=0}^{K-1} 2 \int_{t_k}^{t_k+1}|\dot g(T-t)|^2 \d t e^{2(\delta -g(T-t_{k+1}))} + \dot g(T) e^{2(\delta -g(T))}\Big)  \right)\|\nabla\varphi\|^2_\infty .
    \end{aligned}
\end{equation*}
and combining with \eqref{diff:mstarmT}, we have
\begin{align*}
\E^{(x^i_0)_i\sim m_0^{\otimes N}}&\left[\Big(G(\mu_{T-\delta}^h)-G(m_{\delta}^N)\Big)^2\right]
\lesssim  \|\varphi\|^2_\infty \Big(e^{-g(T)} + \frac{1}{N}\Big) +C_T\|\nabla\varphi\|^2_\infty 
\epsilon^2\\
&+ C_T \|\nabla\varphi\|^2_\infty \mathcal{I}(m^N_{\delta}| m^*)h \Big(\sum_{k=0}^{K-1}  \int_{t_k}^{t_k+1}|\dot g(T-t)|^2 \d t e^{2(\delta -g(T-t_{k+1}))} +\dot g(T) e^{2(\delta -g(T))}\Big).
\end{align*}
Finally, the difference $G(m^N_\delta) - G(m^N_0)$ follows from It\^o's formula, and we obtain the desired estimate. 
\end{proof}

% To propose an optimal scheduling, we need to solve the optimal control problem:
% \[\inf_{g_0=0,~ g_K=T-\delta, g_{k}= \frac{\dot{g}_kT}{K}+g_{k-1},~\dot g_k \uparrow}\sum_{k=1}^{K}|\dot g_{k}|^2  e^{-2 g_k}. \]
% Approximately, we can solve the corresponding continuous time optimal control problem:
% \begin{equation}
% \label{ft_opt_pb}
% \inf_{f_0=0, ~f_1=T-\delta, ~\dot f_t \uparrow} \int_0^1 |\dot f_t|^2 e^{-2 f_t}dt, \quad\mbox{where}\quad f_t = \int_0^t \dot f_s ds.
% \end{equation}

% Solution of the problem (\ref{ft_opt_pb}) is given by the following function:
% \[f_t = - \ln\left(1-t\big(1-e^{-T+\delta}\big)\right), \quad\mbox{for}\quad t\in [0,1].\]

\begin{proof}[Proof of Corollary \ref{cor:ft_opt_pb}]
Denote $\mathcal{G}(g(t)) = \int\limits_0^T | \dot g(t) |^2 e^{ - 2 g(t) } dt$.
Consider a variation $ \delta_t $ of $g(t)$ verifying $  \delta_t(0) = 0, \delta_t(T)=0$. Then we have 
\begin{align*}
& \mathcal{G}( g(t) + \delta_t ) - \mathcal{G}( g(t) ) = A_{ g }( \delta_t ) + o( \delta_t ), \text{ where} \\
& A_{g}( \delta_t ) = \int\limits_0^T 2 \dot g(t) \dot \delta_t e^{ - 2 g(t) } dt - \int\limits_0^T 2 | \dot g(t) |^2 \delta_t e^{ - 2 g(t) } dt. 
\end{align*}
Integrating by parts and taking the boundary conditions into account, we obtain the following first order condition on the solution of the optimization problem $g^*_t$:
\[ \ddot g^*_t - | \dot g^*_t |^2 = 0. \]
Integrating this ODE we obtain that $ g^*_t = - \ln( -t + C_1 ) + C_2 $. Finally, using the boundary conditions we obtain that $ C_1 = \frac{T}{ 1 - e^{-T'} } $, $ C_2 = \ln\left( \frac{T}{ 1 - e^{-T'} } \right) $,
which yields the required formula. 
\end{proof}

\section{Conclusion}
In this paper, we have proved the convergence of denoising diffusion models under a weaker distance than what is usually considered in the literature. More precisely, we have derived an upper-bound on the weak error of a discretization of the denoising process under some regularity assumptions on the test functions used in the computation of the weak errors and regularity assumptions on the noisy score of the empirical measure. Our main result is a \emph{dimensionless} upper bound on the weak convergence error of diffusion models. To the best of our knowledge this bound is new. We show how our proof can be extended to cover the case of discrete diffusion models. Finally, we show that we recover the well-known Flow Matching (or Rectified Flow) schedule by optimizing the upper bound of our weak error analysis. 

While our analysis is tight in the Euclidean case, in future work we would like to extend our weak error bounds to the case of general Denoising Markov Models combining our analysis with the infinitesimal generator framework of \cite{benton2024denoising}. 

\subsubsection*{Acknowledgements} 
A.\,K.'s research is supported by PEPR PDE-AI project. Z.\,R's research is supported by the Finance For Energy Market Research Center,  the France 2030 grant (ANR-21-EXES-0003) and PEPR PDE-AI project.

\bibliography{ref}
\bibliographystyle{tmlr}

\appendix

% \section{Discrete time model}
%  Introduce 
% \[U^h(t, m): = G\big( \mu^h_{T-t}(m)\big).\]
% Note that by the flow property we have
% \[U^h(T, \mu^h_T) = U^h(0, m^*).\]
% Note that for any given $\mu_{t_n}$ and $t\in [t_n, t_{n+1})$, $U^h$ satisfies the frozen-path PDE:
% {\color{red}\[\partial_t U^h + \mathbb{E}^{\overleftarrow{X}^h_{t_n}\sim \mu_{t_n}}\left[ D_m U^h(t, \mu_t, \overleftarrow{X}^h_t ) \big(-b + \frac{\alpha+1}{2}S^*_{T-t_n}\big)(\overleftarrow{X}^h_{t_n}) + \frac{\alpha^2}{2}\nabla\cdot D_m U^h (t, \mu_t, \overleftarrow{X}^h_t )\right] = 0.\]
% (To be verified. Also need to check the differentiability of $U^h$.)} Then it follows from a similar calculus that
% \begin{equation*}
%     \begin{aligned}
%         & G\big(\mu^h_T(m^*)\big) - G\big(m^N_0\big)\\
%         =&~ U^h\big(0, m^*\big) - U^h\big(0, m^N_T\big)\\
%     & - \frac{\alpha+1}{2}\sum_{n=0}^{T/h-1}\int_{t_n}^{t_{n+1}} 
%     \mathbb{E} \left[ D_m U^h \big(t, m^N_{T-t},  X^N_{T-t}\big)  ( \nabla\log m^N_{T-t}\big(t, X^N_{T-t}\big) - S^*_{T-t_n}(X^N_{T-t_n}))\right] dt.
%     \end{aligned}
% \end{equation*}
% {\color{red} it is wrong!...}

\section{Appendix}
 Here we collect the proofs for the technical lemmas.

 \subsection{Proof of Lemma \ref{lem:regularity}}

\begin{proof}[Proof of Lemma \ref{lem:regularity}]
First,  it  follows directly from  \cite[Lemma 5.1]{10.1214/15-AOP1076}  that 
\[D_m U(t,m,x) = \mathbb{E}\left[\partial_x\overleftarrow{X}^{*,t,x}_{T-\delta} \cdot D_m G\big(\mu^*_{T-\delta-t}(m), \overleftarrow{X}^{*,t,x}_{T-\delta}\big) \right],\]
where $\overleftarrow{X}^{*,t,x}$ the process defined in \eqref{eq:backward_eqn_S} starting from $x$ at time $t$.  
Then, the bound of the intrinsic derivative of $U$, $D_m U$,  follows from our Lipschitz assumption on $b$ and $S^*$ together with the boundedness of $D_m G$.

In order to estimate the linear derivative, we perform the following lifting, enabling us to reuse the estimate for the intrinsic derivative.  Define the operator $\mathcal{T}:\mathcal{C}(\R^d)\to\mathcal{C}(\R^{d+1})$ by 
\begin{align*}\varphi(x)\mapsto \mathcal{T}\varphi(x,y)\coloneqq\varphi(x)y,
\end{align*}
and the adjoint operator $\mathcal{T}^*:\mathcal{P}(\R^{d+1})\to\mathcal{P}(\R^{d})$ by
\begin{align*}
    \rho\mapsto \mathcal{T^*}(\rho)(\d x)\coloneqq\int_\R y\rho(\d x\d y).
\end{align*}
Define $\widetilde G:\mathcal{P}(\R^{d+1})\to\R$ by 
$$
\widetilde G(\rho):=G\big(\mathcal{T}^*(\rho)\big).
$$
For $\rho_1,\rho_2\in\mathcal{P}(\R^{d+1})$, it holds, by definition of linear derivative,
\begin{align*}
    \widetilde G(\rho_1)-\widetilde G(\rho_2)&=G\big(\mathcal{T}^*(\rho_1)\big)-G\big(\mathcal{T}^*(\rho_2)\big)
    \\
    &=\int_0^1\Big\langle\frac{\delta G}{\delta m}\big(\lambda \mathcal{T}^*(\rho_1)+(1-\lambda)\mathcal{T}^*(\rho_2),\cdot\big),\mathcal{T}^*(\rho_1)-\mathcal{T}^*(\rho_2)\Big\rangle \d \lambda
    \\
    &=\int_0^1\Big\langle\mathcal{T}\frac{\delta G}{\delta m}\big(\lambda \mathcal{T}^*(\rho_1)+(1-\lambda)\mathcal{T}^*(\rho_2),\cdot\big),\rho_1-\rho_2\Big\rangle \d \lambda .
\end{align*}
Hence, we have for all $\rho\in \mathcal{P}(\R^{d+1})$, 
$$
\frac{\delta \widetilde G}{\delta \rho}(\rho)=\mathcal{T}\frac{\delta G}{\delta m}\big(\mathcal{T}^*\rho\big),
$$
that is for all $x\in\R^d, y\in\R$, we have
\begin{equation}\label{wildeggfirst}
\frac{\delta \widetilde G}{\delta \rho}\big(\rho,(x,y)\big)=\frac{\delta G}{\delta m}\big(\mathcal{T}^*\rho,x\big)y.
\end{equation}
Similarly, we have
\begin{equation}\label{wildeggsecond}
\frac{\delta^2 \widetilde G}{\delta \rho^2}\big(\rho,(x,y),(x',y')\big)=\frac{\delta G}{\delta m}\big(\mathcal{T}^*\rho,x,x'\big)yy'.
\end{equation}
Consider the dynamic:
\begin{equation*}
\begin{cases}
&\d \tilde X_t = \Big(- b(\tilde X_t)  +   \frac{\alpha^2 +1}{2}S^*_{T-t}(\tilde X_t)\Big)dt +  \alpha \d\overleftarrow{W}_t;
\\
&\d Y_t=0.
\end{cases}
\end{equation*}
and $\rho^*_t(\rho)=\text{Law}(\tilde X_t,Y_t)$ given Law$(\tilde X_0, Y_0)=\rho$. Define $\widetilde U:[0,T]\times\mathcal{P}(\R^{d+1})\to\R$ by
$$
\widetilde U(t,\rho)=\widetilde G\big(\rho^*_{T-\delta-t}(\rho)\big).
$$

Let $\rho:=m\otimes\delta_1$, where $\delta_1$ is the Dirac measure at $1$. It is not hard to see that  $\rho^*_t(\rho)=\mu^*_t(m)\otimes\delta_1$, and
$$
\mathcal{T}^*(\rho)=m
,\quad\mathcal{T}^*\big(\rho^*_t(\rho)\big)=\mu^*_t(m)
$$
then
$$
\widetilde U(t,\rho)=\widetilde G\big(\rho^*_{T-\delta-t}(\rho)\big)=G\big(\mu^*_{T-\delta-t}(m)\big)=U(t,m).
$$
Then, by \eqref{wildeggfirst} and \eqref{wildeggsecond}, we have
\begin{align*}
\frac{\delta \widetilde U}{\delta \rho}\big(t,\rho,(x,y)\big)=\frac{\delta }{\delta \rho}\widetilde G\big(\rho^*_{T-\delta-t}(\rho),(x,y)\big)=\frac{\delta}{\delta m}G\big(\mu^*_{T-\delta-t}(m),x\big)y=\frac{\delta U}{\delta m}\big(t,m,x\big)y,
\end{align*}
and similarly
\begin{align*}
\frac{\delta^2 \widetilde U}{\delta \rho^2}\big(t,\rho ,(x,y),(x',y')\big)=\frac{\delta^2 U}{\delta m^2}\big(t,m,x,x'\big)yy'.
\end{align*}
Therefore, we establish the link between the linear derivatives $\frac{\delta U}{\delta m}, \frac{\delta^2 U}{\delta m^2}$ with the intrinsic derivatives $ \partial_y\frac{\delta \widetilde U}{\delta \rho}, \partial^2_{yy'}\frac{\delta^2 \widetilde U}{\delta \rho^2}$:
\[\frac{\delta U}{\delta m}(t,m,x) = \partial_y\frac{\delta \widetilde U}{\delta \rho}(t,m,(x,y)), \quad \frac{\delta^2 U}{\delta m^2}(t,m,x,x') = \partial^2_{yy'}\frac{\delta^2 \widetilde U}{\delta \rho^2}(t,\rho, (x,y),(x',y')).\]
Finally, we again apply \cite[Lemma 5.1, Lemma 5.2]{10.1214/15-AOP1076} to obtain the desired bounds.
\end{proof}

 \subsection{Proof of Lemma \ref{lem:UN}}

Note that
$$
\langle \phi,m_0^N\rangle=\frac{1}{N}\sum_{i=1}^N \phi(X^i),\;\;\langle \phi,m_T^N\rangle=\frac{1}{N}\sum_{i=1}^N\mathbb{E}\big[\phi(X_T)|X_0=X^i\big],
$$
where $X^1,..,X^N$ are i.i.d copies following the law $m_0$.

    \begin{proof}[Proof of Lemma \ref{lem:UN}]
    This essentially repeats the proof of Lemma 5.10 in \cite{10.1214/19-EJP298}. Since $\varphi$ is differentiable, we have the following expansion:
    \begin{align*}
        \varphi(m^N_T)-&\varphi(m_T) =\int_0^1\Big\langle\frac{\delta\varphi}{\delta m}\big(\lambda m_T^N+(1-\lambda)m_T,\cdot\big),m_T^N-m_T\Big\rangle\d \lambda
        \\
        &=\underbrace{\int_0^1\Big\langle\frac{\delta\varphi}{\delta m}\big(\lambda m_T^N+(1-\lambda)m_T,\cdot\big)-\frac{\delta\varphi}{\delta m}\big(m_T,\cdot\big),m_T^N-m_T\Big\rangle\d \lambda}_{S_1}+\underbrace{\Big\langle\frac{\delta\varphi}{\delta m}\big(m_T,\cdot\big),m_T^N-m_T\Big\rangle}_{S_2}.
    \end{align*}
    \textbf{Step 1. } We first deal with the term $S_2$ This may be rewritten as
    \begin{align*}
        S_2=\frac{1}{N}\sum_{i=1}^N\bigg(\E\bigg[\frac{\delta\varphi}{\delta m}\big(m_T,X_T\big)\Big|X_0=X^i\bigg]-\widetilde\E\bigg[\frac{\delta\varphi}{\delta m}\big(m_T,\widetilde X_T\big)\bigg]\bigg),
    \end{align*}
    where under $\widetilde \E$, $\d \widetilde X_t=b(\widetilde X_t)\d t+\d \widetilde W_t$ and Law$(\widetilde X_0)=m_0$. Therefore, we have
    \begin{align*}
        \E[S_2^2]&=\text{Var}\bigg(\frac{1}{N}\sum_{i=1}^N\E\bigg[\frac{\delta\varphi}{\delta m}\big(m_T,X_T\big)\Big|X_0=X^i\bigg]\bigg)=\frac{1}{N^2}\sum_{i=1}^N\text{Var}\bigg(\E\bigg[\frac{\delta\varphi}{\delta m}\big(m_T,X_T\big)\Big|X_0=X^i\bigg]\bigg)
        \\
        &=\frac1N\bigg(\E\bigg[\E\bigg[\frac{\delta\varphi}{\delta m}\big(m_T,X_T\big)\Big|X_0=X^i\bigg]^2\bigg]-\widetilde\E\bigg[\frac{\delta\varphi}{\delta m}\big(m_T,X_T\big)\bigg]^2\bigg)\leq \frac{\big\|\frac{\delta\varphi}{\delta m}\big\|_\infty^2}{N}.
    \end{align*}
    \textbf{Step 2. } We turn to the term $S_1$, which we first rewrite in the form
    \begin{align*}
        S_1 =\frac1N\sum_{i=1}^N\int_0^1\varphi_\lambda^i\d \lambda,
    \end{align*}
    where we define, for $i=1,\cdots,n$ and $\lambda \in [0, 1]$,
    \begin{equation*}
        \begin{aligned}
            \varphi_\lambda^i:=\E\bigg[\frac{\delta\varphi}{\delta m}\big(\lambda m_T^N+&(1-\lambda)m_T,X_T\big)-\frac{\delta \varphi}{\delta m}(m_T,X_T)\Big|X_0=X^i\bigg]
            \\
            -&\widetilde\E\bigg[\frac{\delta\varphi}{\delta m}\big(\lambda m_T^N+(1-\lambda)m_T,\widetilde X_T\big)-\frac{\delta \varphi}{\delta m}(m_T,\widetilde X_T)\bigg],
        \end{aligned}
    \end{equation*}
    where $\widetilde \E$ is taken independently with $m_T^N$ (viewing $m_T^N$ as non-random) and $\widetilde X_T$ is given by the dynamic $\d \widetilde X_t=b(\widetilde X_t)\d t+ \d \widetilde W_t$ with Law$(\widetilde X_0)=m_0$. Now, we have
    $$
    S_1^2\leq \frac{1}{N^2}\int_0^1\bigg(\sum_{i=1}^N\varphi_\lambda^i\bigg)^2\d \lambda=\frac{1}{N^2}\int_0^1\bigg(\sum_{i=1}^N\big(\varphi_\lambda^i\big)^2+2\sum_{i<j}\varphi_\lambda^i\varphi_\lambda^j\bigg)\d \lambda.
    $$
    Note that $|\varphi_\lambda^i|\leq 4\big\|\frac{\delta\varphi}{\delta m}\big\|_\infty$, we have
    $$
    \frac{1}{N^2}\int_0^1\sum_{i=1}^N\big(\varphi_\lambda^i\big)^2\d \lambda\leq \frac{16\big\|\frac{\delta\varphi}{\delta m}\big\|_\infty^2}{N}.
    $$
    It remains to bound $\varphi_\lambda^i\varphi_\lambda^j$ for $i\neq j$. Define $m_T^{N,-i,-j}$ by 
    $$
    \langle\phi,m_T^{N,-i,-j}\rangle:=\frac{1}{N-2}\sum_{k\neq i,j}\mathbb{E}\big[\phi(X_T)|X_0=X^k\big],
    $$
    and $\varphi_\lambda^{i,-j,-k}$ by
    \begin{equation*}
        \begin{aligned}
            \varphi_\lambda^{i,-j,-k}:=\E\bigg[\frac{\delta\varphi}{\delta m}\big(\lambda m_T^{N,-j,-k}+&(1-\lambda)m_T,X_T\big)-\frac{\delta \varphi}{\delta m}(m_T,X_T)\Big|X_0=X^i\bigg]
            \\
            -&\widetilde\E\bigg[\frac{\delta\varphi}{\delta m}\big(\lambda m_T^{N,-j,-k}+(1-\lambda)m_T,\widetilde X_T\big)-\frac{\delta \varphi}{\delta m}(m_T,\widetilde X_T)\bigg].
        \end{aligned}
    \end{equation*}
    Then, it is not hard to see that $\varphi_\lambda^{i,-i,-j}$ is conditionally independent of $\varphi_\lambda^{j,-i,-j}$ given $\{X_k\}_{k\neq i,j}$ and therefore 
    \begin{equation}\label{equ:varphi-i-j0}
    \E\big[\varphi_\lambda^{i,-i,-j}\varphi_\lambda^{j,-i,-j}\big]=0. 
    \end{equation}
    Observe that $\varphi_\lambda^i-\varphi_\lambda^{i,-j,-k}$ can be calculated by 
    \begin{equation}\label{varphiivarphiijldiff}
        \begin{aligned}
            \E\bigg[\frac{\delta\varphi}{\delta m}\big(\lambda m_T^{N}+&(1-\lambda)m_T,X_T\big)-\frac{\delta\varphi}{\delta m}\big(\lambda m_T^{N,-j,-k}+(1-\lambda)m_T,X_T\big)\Big|X_0=X^i\bigg]
            \\
            -&\widetilde\E\bigg[\frac{\delta\varphi}{\delta m}\big(\lambda m_T^{N}+(1-\lambda)m_T,X_T\big)-\frac{\delta\varphi}{\delta m}\big(\lambda m_T^{N,-j,-k}+(1-\lambda)m_T,X_T\big)\bigg],
        \end{aligned}
    \end{equation}
    and
    \begin{align*}
    \frac{\delta\varphi}{\delta m}\big(\lambda m_T^{N}+&(1-\lambda)m_T,X_T\big)-\frac{\delta\varphi}{\delta m}\big(\lambda m_T^{N,-j,-k}+(1-\lambda)m_T,X_T\big)
    \\
    &=\int_0^1\Big\langle\frac{\delta^2\varphi}{\delta m^2}\big(\lambda\gamma m_T^N+\lambda(1-\gamma) m_T^{N,-i,-j}+(1-\lambda)m_T,X_T,\cdot\big),\lambda\big(m_T^N-m_T^{N,-i,-j}\big)\Big\rangle\d \gamma.
    \end{align*}
    Recall that 
    \begin{align*}
    \langle\phi,m_T^{N}-m_T^{N,-i,-j}\rangle:=\frac{1}{N}&\Big(\E\big[\phi(X_T)|X_0=X^i\big]+\E\big[\phi(X_T)|X_0=X^j\big]\Big)
    \\
    &-\frac{2}{N(N-2)}\sum_{k\neq i,j}\mathbb{E}\big[\phi(X_T)|X_0=X^k\big],
    \end{align*}
    we have 
    \begin{align*}
    \bigg|\frac{\delta\varphi}{\delta m}\big(\lambda m_T^{N}+(1-\lambda)m_T,X_T\big)-&\frac{\delta\varphi}{\delta m}\big(\lambda m_T^{N,-j,-k}+(1-\lambda)m_T,X_T\big)\bigg|\leq \frac{2\lambda\big\|\frac{\delta^2\varphi}{\delta m^2}\big\|_\infty}{N}+\frac{2\lambda\big\|\frac{\delta^2\varphi}{\delta m^2}\big\|_\infty}{N-2}
    \\
    &\leq \frac{8\big\|\frac{\delta^2\varphi}{\delta m^2}\big\|_\infty}{N},
    \end{align*}
    where the last inequality holds because without loss of generality, we could assume $N\geq 3$ and $3(N-2)\geq N$. Plugging back into \eqref{varphiivarphiijldiff}, we get
    $$
    \big|\varphi_\lambda^i-\varphi_\lambda^{i,-j,-k}\big|\leq \frac{16\big\|\frac{\delta^2\varphi}{\delta m^2}\big\|_\infty}{N}
    $$
    and together with \eqref{equ:varphi-i-j0}, we have
    \begin{align*}
        \bigg|\E\big[\varphi_\lambda^i\varphi_\lambda^j\big]\bigg|&=\bigg| \E\big[\big|\varphi_\lambda^{i,-i,-j}\varphi_\lambda^{j,-i,-j}\big]+\E\Big[\varphi_\lambda^i(\varphi_\lambda^{j}-\varphi_\lambda^{j,-i,-j}\big)\Big]+\E\Big[\big(\varphi_\lambda^i-\varphi_\lambda^{i,-i,-j}\big)\varphi_\lambda^{j,-i,-j}\Big]\bigg|
        \\
        &\leq \E\Big[\big|\varphi_\lambda^i\big|\big|\varphi_\lambda^{j}-\varphi_\lambda^{j,-i,-j}\big|\Big]+\E\Big[\big|\varphi_\lambda^{j,-i,-j}\big|\big|\varphi_\lambda^{i}-\varphi_\lambda^{i,-i,-j}\big|\Big]\leq \frac{64\big\|\frac{\delta^2\varphi}{\delta m^2}\big\|_\infty\big\|\frac{\delta\varphi}{\delta m}\big\|_\infty}{N}.
    \end{align*}
    Therefore, we have
    $$
    \E\big[S_1^2\big]\leq \frac{16\big\|\frac{\delta\varphi}{\delta m}\big\|_\infty^2}{N}+\frac{128\big\|\frac{\delta^2\varphi}{\delta m^2}\big\|_\infty\big\|\frac{\delta\varphi}{\delta m}\big\|_\infty}{N},
    $$
    and thus
    $$
    \E\Big[\big|\varphi(m_T)-\varphi(m_T^N)\big|^2\Big]=\E\big[\big|S_1+S_2\big|^2\big]\leq \frac{C}{N},$$
    where $C$ depends only on the bounds on the first and second-order derivatives.    
\end{proof}

\subsection{Proof of Lemma \ref{lem:estimate uk}}

\begin{proof}[Proof of Lemma \ref{lem:estimate uk}]
The bound of $\|u^{(k)}\|$ and $\|v^{(k)}\|$ can easily be obtained by its probability representation. As for the gradients' bound, For $t\in[t_k, t_{k+1})$, given $\overleftarrow{X}^h_{t_{k}}$, $\overleftarrow{X}^h_t$ follows an OU process, then we can write 
$$
\overleftarrow{X}^h_{t_{k+1}}\sim\text{N}\bigg(\overleftarrow{X}^h_te^{-g(T-t)+g(T-t_{k+1})}+S_{T-t_k}^*(\overleftarrow{X}^h_{t_{k}})\big(1-e^{-g(T-t)+g(T-t_{k+1})}\big),\frac{1-e^{-2(g(T-t)-g(T-t_{k+1})}}{2}\bigg).
$$
Then,
\begin{align*}
u^{(k)}(t,x;x_k)&=\E\Big[v^{(k+1)}(\overleftarrow{X}^h_{t_{k+1}})\Big|\overleftarrow{X}^h_{t} =x,\overleftarrow{X}^h_{t_k} =x_k\Big]=\E\Big[v^{(k+1)}\big(\sigma^{(k)}(t)Z+\mu^{(k)}(t,x;x_k)\big)\Big],
\end{align*}
and
\begin{align*}
v^{(k)}(x)&=\E\Big[v^{(k+1)}(\overleftarrow{X}^h_{t_{k+1}})\Big| \overleftarrow{X}^h_{t_k} =x\Big]=\E\Big[v^{(k+1)}\big(\sigma^{(k)}(t_k)Z+\mu^{(k)}(t_k,x;x)\big)\Big],
\end{align*}
where $Z\sim$ N$(0,I_d)$ is standard normal and 
$$
(\sigma^{(k)})^2(t)\coloneqq\frac{1-e^{-2(g(T-t)-g(T-t_{k+1}))}}{2}, \;\mu^{(k)}(t,x;x_k)=xe^{-g(T-t)+g(T-t_{k+1})}+S_{T-t_k}^*(x_k)\big(1-e^{-g(T-t)+g(T-t_{k+1})}\big)
$$
Then,
\begin{align*}
\partial_{x_i} &v^{(k)}(x)=\E\bigg[\frac{\d\, v^{(k+1)}\big(\sigma^{(k)}(t_k)Z+\mu^{(k)}(t_k,x;x)\big)}{\d x_i}\bigg]
\\
&=\E\bigg[\sum_{j=1}^d\big(1-e^{-g(T-t_k)+g(T-t_{k+1})}\big)\partial_{x_i}(S^*_{T-t_k}(x))_j\cdot\partial_{x_j}v^{(k+1)}\big(\sigma^{(k)}(t_k)Z+\mu^{(k)}(t_k,x;x)\big)
\\
&\qquad\qquad+e^{-g(T-t_k)+g(T-t_{k+1})}\partial_{x_i}v^{(k+1)}\big(\sigma^{(k)}(t_k)Z+\mu^{(k)}(t_k,x;x)\big)\bigg],
% \\
% &=\E\bigg[Z\bigg(\sum_{j=1}^d\frac{\big(1-e^{g(T-t_{k+1})-g(T-t_k)}\big)\partial_{x_i}(S^*_{T-t_k}(x))_j}{\sigma^{(k)}(t_k)}\cdot v^{(k+1)}\big(\sigma^{(k)}(t_k)Z+\mu^{(k)}(t_k,x;x)\big)
% \\
% &\qquad\qquad+\frac{e^{g(T-t_{k+1})-g(T-t_k)}}{\sigma^{(k)}(t_k)}\cdot v^{(k+1)}\big(\sigma^{(k)}(t_k)Z+\mu^{(k)}(t_k,x;x)\big)\bigg)\bigg]
\end{align*}
which leads to 
\begin{align}\label{tmp1}
\|\nabla v^{(k)}\|_\infty\leq \|\nabla v^{(k+1)}\|_\infty\cdot \Big(e^{-g(T-t_k)+g(T-t_{k+1})}+\big(1-e^{-g(T-t_{k})+g(T-t_{k+1})}\big)\|S^*\|_{\text{Lip},1,\infty}\Big).
\end{align}
By iteration, we get 
\begin{equation}\label{tmp2}
\begin{aligned}
\big\|\nabla v^{(k)}\big\|_\infty&\leq \|\nabla v^{(K)}\|_\infty e^{\sum_{\ell=k}^{K-1}(g(T-t_{k+1})-g(T-t_{k}))}\prod_{\ell=k}^{K-1}\Big(1+\big(e^{g(T-t_\ell)-g(T-t_{\ell+1})}-1\big)\|S^*\|_{\text{Lip}}\Big)
\\
&= \|\nabla \varphi\|_\infty \exp\bigg\{g(T-t_K)-g(T-t_k)+\sum_{\ell=k}^{K-1}\ln\Big(1+\big(e^{g(T-t_\ell)-g(T-t_{\ell+1})}-1\big)\|S^*\|_{\text{Lip},1,\infty}\Big)\bigg\}
\\
&\leq \|\nabla \varphi\|_\infty \exp\bigg\{g(T-t_K)-g(T-t_k)+\|S^*\|_{\text{Lip},1,\infty}\sum_{\ell=k}^{K-1}\big(1-e^{g(T-t_\ell)-g(T-t_{\ell+1})}\big)\bigg\}
\\
&\leq \|\nabla \varphi\|_\infty \exp\bigg\{g(T-t_k)-g(T-t_K)+\|S^*\|_{\text{Lip},1,\infty}\sum_{\ell=k}^{K-1}g(T-t_\ell)-g(T-t_{\ell+1})\bigg\}
\\
&\leq \|\nabla \varphi\|_\infty \exp\bigg\{\Big(g(T-t_k)-g(T-t_K)\Big)\Big(1+\|S^*\|_{\text{Lip},1,\infty}\Big)\bigg\},
\end{aligned}
\end{equation}
which leads to the desired bound of $\|\nabla v^{(k)}\|_\infty.$
Using the similar argument in \eqref{tmp1} , we can also get for all $y\in\R^d$,
\begin{align*}
\partial_{x_i} &u^{(k)}(t,x;y)=\E\bigg[\frac{\d\, v^{(k+1)}\big(\sigma^{(k)}(t)Z+\mu^{(k)}(t,x;y)\big)}{\d x_i}\bigg]
\\
&=\E\Big[e^{g(T-t)-g(T-t_{k+1})}\partial_{x_i}v^{(k+1)}\big(\sigma^{(k)}(t)Z+\mu^{(k)}(t_k,x;x_k)\big)\Big],
% \\
% &=\E\bigg[Z\bigg(\sum_{j=1}^d\frac{\big(1-e^{g(T-t_{k+1})-g(T-t_k)}\big)\partial_{x_i}(S^*_{T-t_k}(x))_j}{\sigma^{(k)}(t_k)}\cdot v^{(k+1)}\big(\sigma^{(k)}(t_k)Z+\mu^{(k)}(t_k,x;x)\big)
% \\
% &\qquad\qquad+\frac{e^{g(T-t_{k+1})-g(T-t_k)}}{\sigma^{(k)}(t_k)}\cdot v^{(k+1)}\big(\sigma^{(k)}(t_k)Z+\mu^{(k)}(t_k,x;x)\big)\bigg)\bigg]
\end{align*}
which means for all $y\in\R^d$,
$$
\big\|\nabla u^{(k)}(t,\cdot;y)\big\|_\infty\leq e^{g(T-t)-g(T-t_{k+1})} \big\|\nabla v^{(k+1)}\big\|_\infty.
$$
Combining with \eqref{tmp2}, we have for all $y\in\R^d$.
\begin{equation}\label{nablaubdd}
\begin{aligned}
    \big\|\nabla u^{(k)}(t,\cdot;y)\big\|_\infty
    &\leq \|\nabla \varphi\|_\infty e^{(1+\|S^*\|_{\text{Lip},1,\infty})g(T)}.
\end{aligned}
\end{equation}
which completes the proof of the lemma.
\end{proof}

\end{document}